\documentclass[11pt]{article}

\usepackage[
  shownumpages,               
  bgcolor={245,245,250},      
  braincolor={190,32,240},    
  citingstyle=authoryear,     
]{brainlab}

\usepackage{mathtools}
\usepackage{nicefrac}

\usepackage{url}            
\usepackage{booktabs}       
\usepackage{amsfonts}       
\usepackage{nicefrac}       
\usepackage{microtype}      
\usepackage{lipsum}		
\usepackage{graphicx}
\usepackage{wrapfig}
\usepackage{doi}
\usepackage{mathtools}
\usepackage{xspace}	
\usepackage{amsthm}
\usepackage{setspace}
\usepackage{comment}
\usepackage{amsmath,amssymb}
\usepackage{algorithm}
\usepackage{algpseudocode}
\usepackage{tablefootnote}
\usepackage{array, longtable}
\usepackage{bbm}
\usepackage{epigraph}
\usepackage{enumitem}
\usepackage{float}
\usepackage{subcaption}
\usepackage{booktabs}
\usepackage{tabularx}
\usepackage{multirow}
\usepackage{makecell}
\usepackage{subcaption}
\usepackage{threeparttable}
\usepackage{wasysym}

\definecolor{myblue}{RGB}{0, 0, 139}
\definecolor{mylightblue}{RGB}{100, 149, 237}

\newtheorem{lemma}{Lemma}

\newcommand{\Rd}{\mathbb{R}^d}
\newcommand{\R}{\mathbb{R}}
\newcommand{\E}{\mathbb{E}}

\setbrainmeta{
  title={Enhancing Stability of Physics-Informed Neural Network Training Through Saddle-Point Reformulation},
  authors={
    Dmitry Bylinkin, Mikhail Aleksandrov\textsuperscript{2}, Savelii Chezhegov, Aleksandr Beznosikov\textsuperscript{1}
  },
  affiliations={
    \textsuperscript{1}Innopolis University\\
    \textsuperscript{2}Moscow State University
  },
  abstract={
    Physics-informed neural networks (\textit{PINN}s) have gained prominence in recent years and are now effectively used in a number of applications. However, their performance remains unstable due to the complex landscape of the loss function. To address this issue, we reformulate \textit{PINN} training as a nonconvex-strongly concave saddle-point problem. After establishing the theoretical foundation for this approach, we conduct an extensive experimental study, evaluating its effectiveness across various tasks and architectures. Our results demonstrate that the proposed method outperforms the current state-of-the-art techniques. 
  },
}

\begin{document}
\maketitlebox
\section{Introduction}\label{sec:intro}
Mathematical physics is a cornerstone of modern science. It provides powerful tools for theoretical studies and finds applications in a wide range of practical fields. One of its central challenges is solving partial differential equations (PDEs) \citep{bateman1932partial, evans2022partial}. They arise in the formal description of phenomena ranging from heat diffusion to quantum mechanics and typically take the form of a boundary value problem involving differential operators on some domain \citep{yakubov1999differential}. Generally, there is a system
\begin{align}\label{eq:problem}
\begin{split}
    &\quad\mathcal{R}_i[u](x)=f_i(x),~i\in[1,M_r],~x\in\Omega;\\
    &\mathcal{B}_j[u](x)=g_j(x),~j\in[M_{r+1},M],~x\in\partial\Omega,
\end{split}
\end{align}
where $f_i,g_i:\Rd\to\R$ are the scalar functions; $\mathcal{R}_i[u],\mathcal{B}_i[u]:\R^d\to\R$ are the operators actions on the mapping $u:\Rd\to\R^m$; $\Omega\subset\R^d$ and $\partial\Omega\subset\R^{d-1}$ are the domain set and its boundary, respectively. Since exact solutions are rare outside idealized cases, the community is focused on developing numerical methods. Among the most established techniques are those based on finite differences \citep{courant1967partial}, volumes \citep{patankar1983calculation}, and elements \citep{courant1994variational}. Despite their widespread use, traditional approaches suffer from several limitations. First, they require solving a system of linear equations at each iteration \citep{saad2000iterative} and are therefore computationally complex. Moreover, specifying a grid poses challenges when the geometry of $\Omega$ is non-trivial \citep{babuvska1976angle, duster2008finite}. Finally, they either lack stability or exhibit poor approximation quality \citep{strang1971analysis, arnold2002unified}. These factors prompt researchers to explore alternatives, including machine learning techniques that have emerged in this field \citep{guo2016convolutional, zhu2018bayesian, yu2018deep}. Although the concept of approximating the solution with a parametrized function $u(\theta)$ is quite old and dates back to the works of \citet{meade1994numerical, dissanayake1994neural, lagaris1998artificial}, it has only recently gained attention under the name \textit{PINN (physics-informed neural network)} \citep{raissi2019physics}. Neural networks demonstrate good performance in approximating complex dependencies, including high-dimensional problems \citep{cybenko1989approximation, cheridito2021efficient}. Moreover, they appear promising in overcoming the aforementioned bottlenecks of grid-based numerical approaches \citep{li2020fourier}. While initial results in this area were obtained using \textit{MLP}s, advanced architectures such as learned activations \citep{jagtap2020locally, jagtap2020adaptive}, memory \citep{krishnapriyan2021characterizing, cho2023hypernetwork} and attention \citep{zhao2023pinnsformer, anagnostopoulos2024residual} have led to significant improvements. Typical of AI-based solutions, \textit{PINN}s are trained through empirical risk minimization (ERM) \citep{raissi2019physics}:
\begin{align*}
    \min_{\theta\in\R^d}&\left[\mathcal{L}(\theta)=\sum_{i=1}^{M_r}\mathcal{L}_{r,i}(\theta)~+\!\sum_{j={M_r+1}}^{M}\!\mathcal{L}_{r,j}(\theta)\right],\\
    &\text{with }\mathcal{L}_{r,i}(\theta)=\frac{1}{N_{r}}\sum_{n=1}^{N_{r}}\left[\mathcal{R}_i[u(\theta)](x^n_r)-f(x_r^n)\right],\\
    &\quad\quad\mathcal{L}_{b,j}(\theta)=\frac{1}{N_{b}}\sum_{n=1}^{N_{b}}\left[\mathcal{B}_j[u(\theta)](x_b^n) -g(x_b^n)\right],
\end{align*}
where $\{x_r^n\}_{n=1}^{N_r},~\{x_r^j\}_{j=1}^{N_b}$ are the sets of samples belonging to the interior and boundary of $\Omega$, respectively; $N_r,~N_b$ are the sizes of the corresponding datasets.

Despite the successes, \textit{PINN}s bring their own challenges. Training them via solving the problem \eqref{eq:problem} is a special case of multi-task learning \citep{zhang2021survey}. Indeed, a single model is trained to approximate all the operators simultaneously. However, they may be of a different nature. Hence, there is no guarantee that $\arg\min_{\theta\in\R^d}\mathcal{L}(\theta)$ minimizes all $\mathcal{L}_{r,i}(\theta)$ and $\mathcal{L}_{b,j}(\theta)$ individually. In practice, their corresponding gradients $\nabla\mathcal{L}_{r,i}(\theta)$, $\nabla\mathcal{L}_{b,j}(\theta)$ have dissimilar magnitudes (see Figure 2 in \citep{hwang2024dual}). Consequently, some losses are ignored during optimization. As a result, the solution is well approximated only on the boundary or only inside the domain (see Figure 1 in \citep{hwang2024dual}). Moreover, the loss landscape is poor \citep{krishnapriyan2021characterizing}. This makes it impossible to guarantee the efficient performance of basic optimizers. Despite significant interest in the area, there remains no universally effective approach for training \textit{PINN}s. A scheme that performs well for one PDE may turn out to be inadequate for another (see Table 3 in \citep{hao2023pinnacle}). Selecting an appropriate optimizer often requires case-by-case tuning. Nevertheless, one can distinguish a class of approaches that are particularly successful for training neural networks for the problem \eqref{eq:problem}. A common theme is the use of coefficients $\pi=(\pi_1,\ldots,\pi_M)^\top$ to balance competing losses for $\mathcal{R}_i[u]$, $\mathcal{B}_j[u]$. Numerous approaches to weighting are known in the literature \citep{wang2021understanding, jin2021nsfnets, wang2022and, son2023enhanced, hwang2024dual}. In our work, we consider training \textit{PINN} as a saddle-point problem (SPP):
\begin{align}\label{eq:pinn_saddle}
    \min_{\theta\in\R^d}\max_{\pi\in S}\left[\mathcal{L}(\theta, \pi)=\sum_{i=1}^{M_r}\pi_i\mathcal{L}_{r,i}(\theta)~+\!\sum_{j={M_r+1}}^{M}\!\pi_j\mathcal{L}_{r,j}(\theta) - \lambda D_{\psi}(\pi||\hat{\pi}) \right],
\end{align}
where $D_{\psi}(\cdot||\hat{\pi})$ is the Bregman divergence \citep{nemirovskij1983problem}. In some cases, it offers a better description of distance than the Euclidean metric. For example, if $S$ is a probabilistic simplex, then $D_{\psi}(\pi,\hat{\pi})=\text{KL}(\pi,\hat{\pi})\coloneqq\sum_{i=1}^M\pi_i\ln\left(\nicefrac{\pi_i}{\hat{\pi}_i}\right)$. This method of measuring distances is preferable, particularly because it accounts for the relative rather than absolute change in weights. The use of $D_{\psi}(\cdot||\hat{\pi})$ to regularize \eqref{eq:pinn_saddle} prevents overemphasizing the importance of any operator. A similar methodology was considered in \citep{liu2021dual}. However, the authors provided no theoretical guarantees and examined the Euclidean case, which is unsuitable due to the complex geometry of $S$. At the same time, choosing $S=\Rd$ leads to unstable training, as the weights in such a case can grow indefinitely, making this approach impractical. Thus, the problem \eqref{eq:pinn_saddle} remains unexplored for \textit{PINN}s. In this work, we overcome both theoretical and practical challenges to investigate the feasibility of training physical neural networks as SPPs.
\section{Related Works} 

\subsection{Loss Rescaling in General Case}

Earlier, we mentioned that training a physics-informed neural network is a special case of multi-task learning, where various rescaling techniques had been developed by the time of the emergence of \textit{PINN}s. \citet{chen2018gradnorm} suggested treating the weights as trainable functions $\pi_m(\hat\theta)$. They defined a separate loss such that the norm of a single task gradient $\nabla(\pi_m(\hat{\theta})\mathcal{L}_{\_,i}(\theta))$ is close to the sum of the other gradients. A similar approach was explored in \citep{kendall2018multi}. However, using neural networks to evaluate the parameters leads to increased memory consumption. As a consequence, the community has developed a number of computationally less expensive techniques. \citet{sener2018multi} proposed solving a quadratic optimization problem on a unit simplex to determine $\{\pi_m\}_{m=1}^M$. Furthermore, approaches that calculate weights via zero- and first-order statistics have gained attention due to their combination of efficiency and quality \citep{liu2019end, yu2020gradient, heydari2019softadapt, chen2018gradnorm, wang2020gradient}. 
\subsection{Loss Rescaling in \textit{PINN}s}

The unique challenges posed by PDEs and physical constraints motivated the development of weighting techniques specifically for \textit{PINN}s. \citet{wang2021understanding} were among the first in this direction. Inspired by ideas behind \texttt{Adam} \citep{kingma2014adam}, they proposed a learning rate annealing procedure that automatically tunes $\{\pi_m\}_{m=1}^M$ by utilizing the back-propagated gradient statistics. To mitigate the high variance inherent in the stochastic nature of updates, the authors suggested computing the actual weights as a running average of their previous values. This scheme was then understood in greater depth \citep{jin2021nsfnets, maddu2022inverse, bischof2025multi}. As an orthogonal approach, in \citep{wang2022and}, loss rescaling was addressed from a neural tangent kernel perspective. Despite the advances, it may be computationally expensive. Indeed, the use of the Jacobian poses a challenge when solving nonlinear equations, as it is not constant in that case \citep{bonfanti2024challenges}. In parallel to these commonly used approaches, a number of exotic non-benchmarked techniques exist. For example, schemes based on likelihood \citep{xiang2022self, hou2023enhancing}, augmented Lagrangian \citep{son2023enhanced} and conjugate cone \citep{hwang2024dual}. 

\subsection{Nonconvex-Strongly Concave SPPs}
The theory of SPPs is constructed mostly for convex-concave objectives \citep{korpelevich1976extragradient, nemirovski2004prox, du2019linear, adolphs2019local, beznosikov2023unified}. However, the problem \eqref{eq:pinn_saddle} falls outside of this class. Indeed, complex nature of differential operators implies a poor non-convex landscape in $\theta$. On the other hand, in terms of the weights $\pi$, $\mathcal{L}(\theta,\pi)$ is a regularized linear function, and hence is guaranteed to be strongly concave regardless of the PDE being solved. Nonconvex-concave (N-C) and nonconvex-strongly concave (N-SC) SPPs remain poorly understood. Today's research focuses on modifying two-timescale gradient descent-ascent (\texttt{TT-GDA}), which has demonstrated success in training GANs \citep{heusel2017gans}. Using a double-loop scheme, \citet{nouiehed2019solving} achieved a $\varepsilon$-solution in $\tilde{\mathcal{O}}\left(\nicefrac{\kappa^4}{\varepsilon^2}\right)$ iterations, where $\kappa$ denotes the condition number of the objective in the concave component. Assuming $\max$-oracle to be available, \citet{jin2019minmax} improved this result to $\tilde{\mathcal{O}}\left(\nicefrac{\kappa^2}{\varepsilon^2}\right)$. In parallel, several triple-loop techniques for N-C problems were developed \citep{thekumparampil2019efficient, kong2021accelerated}. 
However, algorithms with nested loops are challenging to implement and tune in practice. This is supported by the observation that the mentioned papers consider simple problems (e.g. classification on \textit{MNIST}) for their experiments. At the same time, providing a theoretical analysis directly to \texttt{TT-GDA} posed a challenge. This was finally done in \citep{lin2020gradient} with a complexity of $\mathcal{O}\left(\nicefrac{\kappa^2}{\varepsilon^2}\right)$. Later, the result was generalized by \citet{xu2023unified}. They provided unified analysis of single-loop schemes for N-C problems. 

A key drawback of the mentioned methods is the Euclidean setting. This may be inappropriate for describing the geometry of $S$ in the problem \eqref{eq:pinn_saddle}, as it is typically defined as a bounded set to maintain balance during training \citep{mohri2019agnostic, mehta2024drago}. Consequently, there is interest in searching for alternatives. \citet{huang2021efficient} considered a setup that is non-Euclidean in the non-convex component and Euclidean in the strongly concave one. However, in our paper, we need the opposite. Indeed, in the problem \eqref{eq:pinn_saddle}, $\theta$ lies in $\Rd$ and is therefore suited to the Euclidean distance, while $\pi$ demands a more complicated description. Thus, this work is not suitable for our purposes, although it provides useful intuition. \citet{boroun2023projection} employed Frank-Wolfe \citep{jaggi2013revisiting} to perform both ascent and descent steps. However, exploiting non-regularized linear approximation yields sparse values of $\{\pi_m\}_{m=1}^M$, which may result in unstable convergence.

\section{Our Contribution}
We suggest that \textit{PINN}s could achieve better results when 
trained as the saddle-point problem of the form \eqref{eq:pinn_saddle}. The paper presents a comprehensive theoretical and empirical analysis of this approach. Our key contributions are:
\begin{itemize}
    \item \textbf{Theoretical foundation.} Studying nonconvex-strongly concave SPPs with non-Euclidean geometry of the strongly concave component, we propose a method based on a suitable Bregman proximal mapping. We develop a rigorous theory, providing guarantees on optimization dynamics.
    \item \textbf{Extensive Empirical Validation.} Conducting experiments on benchmark PDEs, we demonstrate that our approach improves the quality compared to existing optimizers and enhances training stability (See Table \ref{tab:comparison_contr} for some of the results). Once being tuned, the proposed algorithm achieves SOTA results on almost all benchmark PDEs. At the same time, previously, each task had its own dominant method for each type of problem: \texttt{LRA} \citep{wang2021understanding} for \textit{Poisson} and \textit{High dim}, \texttt{RAR} \citep{lu2021deepxde} for \textit{Heat}, \texttt{NTK} \citep{wang2022and} for \textit{Wave}, and \texttt{Adam} \citep{kingma2014adam} for \textit{Navier-Stokes}.
    \item \textbf{Practical Insights.} We provide in-depth guidelines for practical application of our method.
\end{itemize}
\begin{table}[htbp]
\centering
\renewcommand{\arraystretch}{1.3}
\setlength{\tabcolsep}{0pt} 
\begin{tabularx}{\linewidth}{@{} l *{5}{>{\centering\arraybackslash}X} @{}}
\toprule
\textbf{Approach} & \textbf{Poisson} & \textbf{Heat} & \textbf{Navier-Stokes} & \textbf{Wave} & \textbf{High dim} \\
\midrule
Previous best & 1.02E-1 & 2.72E-2 & 4.70E-2 & 9.79E-2 & 4.58E-4 \\
\rowcolor{myblue!20}
This paper & \textbf{4.78E-2} & \textbf{1.01E-2} & \textbf{2.24E-2} & \textbf{1.62E-2} & \textbf{1.20E-4} \\
\bottomrule
\end{tabularx}
\caption{Comparison of SOTA results with the proposed approach. \textbf{L2RE} is used as a quality metric (lower values indicate better performance).}
\label{tab:comparison_contr}
\end{table}
\vspace{-0.5cm}
\section{Setup}
\subsection{Assumptions}
Since our study is motivated by the real-world problem, we avoid introducing unnecessary limitations and address the most general case possible. First, we require the objective to be smooth with respect to the Euclidean norm. 
\begin{assumption}\label{ass:theta}
    The function $\mathcal{L}(\theta,\pi)$ is $L$-smooth, i.e. for all $(\theta_1,\pi_1), (\theta_2,\pi_2) \in \Rd\times S$ it satisfies 
    \begin{align*}
        \|\nabla \mathcal{L}(\theta_1,\pi_1) - \nabla \mathcal{L}(\theta_2,\pi_2)\|^2 \leq L^2\left(\|\theta_1-\theta_2\|^2+\|\pi_1-\pi_2\|^2\right).
    \end{align*}
\end{assumption}
This is the standard assumption that is widely spread in the optimization literature \citep{adolphs2019local, lin2020gradient}.
To enable more accurate selection of the weights $\pi$, we account for the geometry of $S$ by utilizing the Bregman divergence \citep{nemirovskij1983problem}.
\begin{definition}\label{def:Bregman}
    The Bregman divergence corresponding to the \textit{distance generating function} $\psi:S\to\R$ is defined as
    \begin{align*}
    D_{\psi}(\pi_1,\pi_2) = \psi(\pi_1) - \psi(\pi_2) -\langle \nabla \psi(\pi_2), \pi_1-\pi_2\rangle,~\forall \pi_1, \pi_2 \in S.
    \end{align*}
\end{definition}
Earlier, we mentioned the example where $D_{\psi}$ is the Kullback-Leibler divergence. This is particularly significant for the purposes of this paper, as we choose $S$ as the unit simplex. However, the theory is established in the general case. Analysis of the problem \eqref{eq:pinn_saddle} requires $D_{\psi}$ to have several basic properties. In particular, Definition \ref{def:metric} is valid only if $D_{\psi}$ is bounded from below on $S$. In the following, we present an assumption regarding the distance generating function.
\begin{assumption}\label{ass:p}
        The function $\psi$ is \textbf{1-strongly convex}, i.e. for all $\pi_1,\pi_2\in S$ it satisfies
        \begin{align*}
            \psi(\pi_1) \geqslant \psi(\pi_2) + \left\langle \nabla \psi(\pi_2), \pi_1 - \pi_2 \right\rangle + \frac{1}{2}\|\pi_2 - \pi_1\|^2.
        \end{align*}
\end{assumption}
Note that this assumption does not reduce the class of neural networks under consideration, as it is solely related to the choice of regularizer. Additionally, it holds for all commonly used divergences.

\subsection{Properties of the Objective}
The problem \eqref{eq:pinn_saddle} is a special case of nonconvex-strongly concave SPPs. In this section, we obtain several properties of the objective by leveraging its structure. Firstly, we formulate the following.
\begin{lemma}\label{eq:lemma_1}
    Consider the problem \eqref{eq:pinn_saddle} under Assumption \ref{ass:p}. Then, for every $\theta\in\Rd$ the function $\mathcal{L}(\theta,\pi)$ is \textbf{$\lambda$-strongly concave}, i.e. for all $\pi_1,\pi_2\in S$ it satisfies
    \begin{align*}
        \mathcal{L}(\theta,\pi_1)\leq\mathcal{L}(\theta,\pi_2)+\langle\nabla_{\psi}\mathcal{L}(\theta,\pi_2),\pi_1-\pi_2\rangle-\frac{\lambda}{2}\left(D_{\psi}(\pi_1,\pi_2)+D_{\psi}(\pi_2,\pi_1)\right).
    \end{align*}
\end{lemma}
See the proof in Appendix \ref{ap:A}. Thus, Lemma \ref{eq:lemma_1} in combination with Assumption \ref{ass:theta} shows that the problem \eqref{eq:pinn_saddle} is indeed a nonconvex-strongly concave SPP. Moreover, Assumption \ref{ass:p} entails strong concavity of $\mathcal{L}(\theta,\pi)$ in $\pi$. Consequently, it has a single maximum $\pi^*(\theta)$ on $S$ for every fixed value of $\theta$.
\subsection{Notation}
Since N-SC SPPs have a complex structure, it is technically difficult to prove the convergence of the method using the usual definition of a stationary point. Instead, the definition is equivalently reduced to a stationary point of a minimization problem \citep{huang2021efficient}. Let us consider
\begin{align*}
    \Phi(\theta)=\mathcal{L}(\theta,\pi^*(\theta)).
\end{align*}
Since $S$ is the bounded convex set, Danskin's theorem implies that $\Phi$ is differentiable with $\nabla\Phi(\theta)=\nabla_{\theta}\mathcal{L}(\theta,\pi^*(\theta))$ \citep{rockafellar2015convex}. The common convergence metric is the following.
\begin{definition}\label{def:metric}{\textbf{($\varepsilon$-stationary point) of $\Phi(\theta)$.}}
A point $\theta$ is a $\varepsilon$-stationary point of $\Phi$, if
\begin{align*}
    \|\nabla\Phi(\theta)\|\leq\varepsilon.
\end{align*} 
\end{definition}
In this paper, $\|\cdot\|$ is understood as a Euclidean norm.

\section{Algorithm and Analysis}
In this section, we follow the trend of investigating N-SC SPPs through modifications of \texttt{TT-GDA}. Adapting it to the problem \eqref{eq:pinn_saddle}, we present \textbf{B}regman \textbf{G}radient \textbf{D}escent \textbf{A}scent.
\begin{algorithm}{\texttt{BGDA}} \label{alg:bgda}
\begin{algorithmic}[1]
    \State {\bfseries Input:} Starting point $(\theta^0,\pi^0)\in\Rd\times S$, number of iterations $T$
    \State {\bfseries Parameters:} Stepsizes $\gamma_{\theta},\gamma_{\pi} > 0$
    \For{$t=0,\ldots, T-1$}
    \State $\theta^{t+1}=\theta^t-\gamma_{\theta}\nabla_{\theta}\mathcal{L}(\theta^t,\pi^t)$\label{line:descent} \Comment{Optimizer updates parameters}
    \State $\pi^{t+1}=\arg\min_{\pi\in S}\left\{ -\gamma_{\pi}\left\langle \nabla_{\pi}\mathcal{L}(\theta^t,\pi^t),\pi \right\rangle +D_{\psi}(\pi,\pi^t) \right\}$\label{line:ascent} \Comment{Optimizer updates weights}
    \EndFor
    \State {\bfseries Output:} $(\theta^T,\pi^T)$
\end{algorithmic}
\end{algorithm}
Due to the complex landscape of the problem to be solved, the algorithmic schemes we rely on are extremely simple. Since the parameters $\theta$ may take any value, it suffices to use the classic gradient descent step \citep{nemirovskij1983problem} to update them (Line \ref{line:descent}). However, the weights are selected from a convex bounded set described by Non-Euclidean geometry. Consequently, we utilize the Bregman proximal mapping \citep{nemirovskij1983problem} to perform the ascent step (Line \ref{line:ascent}). The subproblem in Line \ref{line:ascent} requires estimating statistics of the objective only once and therefore does not pose any significant computational difficulties compared to the basic descent step. Moreover, it often has a closed-form solution. For example, if $D_{\psi}$ is the KL-divergence \citep{nemirovskij1983problem}, then
\begin{align*}
    \pi^{t+1}=\left(\frac{\exp\{\gamma_{\pi}(\nabla_{\pi}\mathcal{L}(\theta^t,\pi^r))_i\}}{\sum_{i=1}^M\exp\{\gamma_{\pi}(\nabla_{\pi}\mathcal{L}(\theta^t,\pi^r))_i\}}\right)_{i=1}^M.
\end{align*}

In the analysis of Algorithm \ref{alg:bgda}, it is fundamental to utilize steps of varying sizes. One possible explanation is that the landscape of the objective is much better in the strongly concave component. Consequently, more confident steps can be taken to update the weights. The primary theoretical challenge in the analysis of the method is to show the convergence of the iterative scheme based on the metric given in Definition \ref{def:metric}. Indeed, for each value of the model parameters $\theta^t$ there is an optimal point $\pi^*(\theta^t)$. To address the technical difficulties, we must show that the method generates a sequence of points $\{(\theta^t,\pi^t)\}_{t=1}^T$ for which the distance between $\pi^t$ and $\pi^*(\theta^t)$ decreases when increasing $t$. Moreover, we have to account for the non-Euclidean geometry of the problem.
\begin{lemma}\label{lemma:distance}
    Consider the problem \eqref{eq:pinn_saddle} under Assumptions \ref{ass:theta}, \ref{ass:p}. Then, Algorithm \ref{alg:bgda} produces such $\{(\theta^t,\pi^t)\}_{t=1}^T$, that
    \begin{align*}
        D_{\psi}(\pi^*(\theta^{t+1}),\pi^{t+1})\leq\left(1-\frac{1}{64\kappa^2}\right)D_{\psi}(\pi^*(\theta^t),\pi^{t})+264\gamma_{\theta}^2\kappa^6\|\nabla\Phi(\theta^t)\|^2,
    \end{align*}
    where $\kappa=\nicefrac{L}{\lambda}$ is the condition number of $\mathcal{L}(\theta,\pi)$ in $\pi$.
\end{lemma}
See the proof in Appendix \ref{ap:A}. Lemma \ref{lemma:distance} shows how the distance between the current weight iterate $\pi^t$ and the ideal response $\pi^*(\theta^t)$ evolves over time. This is a key result needed to prove convergence. Indeed, since we consider the Euclidean setting in the nonconvex variables $\theta$, the standard inexact gradient descent analysis implies
\begin{align*}
    \Phi(\theta^{t+1})-\Phi(\theta^0)\leq-\Omega(\gamma_{\theta})\left(\sum_{t=1}^{T-1}\|\nabla\Phi(\theta^{t})\|^2\right)+\mathcal{O}\left(\gamma_{\theta}L^2\right)\sum_{t=1}^{T-1}D_{\psi}(\pi^*(\theta^t),\pi^t).
\end{align*}
Thus, for a sufficiently small step $\gamma_{\theta}$, it is guaranteed to neglect the inaccuracy of finding the maximum at the ascent step. By carefully evaluating $D_{\psi}(\pi^*(\theta^t),\pi^t)$ from above and selecting appropriate $\gamma_{\theta}$, the convergence is obtained. We formulate this fact as a main theorem.
\begin{theorem}\label{th:general}
    Consider the problem \eqref{eq:pinn_saddle} under Assumptions \ref{ass:theta}, \ref{ass:p}. Then, Algorithm \ref{alg:bgda} requires
    \begin{align*}
        \mathcal{O}\left(\frac{\kappa^4L\Delta+\kappa^2L^2D_{\psi}(\pi^*(\theta^0),\pi^0)}{\varepsilon^2}\right)\text{ iterations}
    \end{align*}
    to achieve an arbitrary $\varepsilon$-solution, where $\varepsilon^2 = \frac{1}{T}\sum_{t=1}^{T-1}\|\nabla\Phi(\theta^t)\|^2$, $\Delta=\Phi(\theta^0)-\Phi(\theta^*)$. $\kappa=\nicefrac{L}{\lambda}$ is the condition number of $\mathcal{L}(\theta,\pi)$ in $\pi$.
\end{theorem}
See the proof in Appendix \ref{ap:B}. Note that the derived estimate of $T$ is worse than that obtained in \citep{huang2021efficient} for the Euclidean setting. However, if $S$ is a unit simplex intersected with a euclidean ball, it can be significantly improved (see Appendix \ref{ap:C}) for the detailed discussion. The question of improveability in the general case remains open. After examining a large number of proof approaches, we believe that for \texttt{GDA}-like schemes, it is unimprovable.

\section{Numerical Experiments}\label{sec:exps}
We now present the empirical analysis of our approach. All experiments are conducted on a Linux server utilizing an NVIDIA GeForce RTX 2080 Ti with 12 GB of GPU memory and NVIDIA TESLA A100 with 80 GB of GPU memory. To ensure accurate results, only one card is employed within a single subsection. We also attach the \href{https://anonymous.4open.science/r/pinns-bgda-00D6/}{code} of the experiments to verify reproducibility.

Following the trends of machine learning, we develop an adaptive modification of Algorithm \ref{alg:bgda}. In Algorithm \ref{alg:adaptive_bgda},
\begin{algorithm}{\texttt{Adaptive BGDA}} \label{alg:adaptive_bgda}
\begin{algorithmic}[1]
    \State {\bfseries Input:} Starting point $(\theta^0,\pi^0)\in\Rd\times S$, number of iterations $T$
    \State {\bfseries Parameters:} Stepsizes $\gamma_{\theta},\gamma_{\pi} > 0$
    \For{$t=0,\ldots, T-1$}
    \State $m_{\theta}^{t+1}=\alpha_{1}m_{\theta}^t+(1-\alpha_1)\nabla_{\theta}\mathcal{L}(\theta^t,\pi^t)$\label{line:theta_smoothing}
    \State $v_{\theta}^{t+1}=\alpha_{2}v_{\theta}^t+(1-\alpha_2)\|\nabla_{\theta}\mathcal{L}(\theta^t,\pi^t)\|^2$\label{line:theta_history}
    \State $v_{\pi}^{t+1}=\beta v_{\theta}^t+(1-\beta)\|\nabla_{\pi}\mathcal{L}(\theta^t,\pi^t)\|^2$\label{line:pi_history}
    \State $\widehat{m}_{\theta}^{t+1}=\frac{m_{\theta}^{t+1}}{1-\alpha_1^t}$\label{line:cor_1}
    \State $\widehat{v}_{\theta}^{t+1}=\frac{v_{\theta}^{t+1}}{1-\alpha_2^t}$\label{line:cor_2}
    \State $\widehat{v}_{\pi}^{t+1}=\frac{v_{\pi}^{t+1}}{1-\beta^t}$\label{line:cor_3}
    \State $\theta^{t+1}=\theta^t-\gamma_{\theta}\widehat{m}_{\theta}^{t+1}$\label{line:theta_step} \Comment{Optimizer updates parameters}
    \State $\pi^{t+1}=\arg\min_{\pi\in S}\left\{ -\gamma_{\pi}\left\langle \widehat{m}_{\pi}^{t+1},\pi \right\rangle +D_{\psi}(\pi,\pi^t) \right\}$\label{line:pi_step}\Comment{Optimizer updates weights}
    \EndFor
    \State {\bfseries Output:} $(\theta^T,\pi^T)$
\end{algorithmic}
\end{algorithm}
the gradient $\nabla_{\theta}\mathcal{L}(\theta^t,\pi^t)$ is smoothed with its previous values as a running average (Line \ref{line:theta_smoothing}). In practice, this approach aids in identifying a suitable descent direction more quickly. Furthermore, we propose accumulating the gradient history to vary the step sizes (Lines \ref{line:theta_history}, \ref{line:pi_history}). This method is effective, as the gradient magnitude indicates the loss smoothness locally, which leads to more confident steps and faster convergence. A practice-driven bias correction of calculated statistics is also implemented (Lines \ref{line:cor_1}, \ref{line:cor_2}, \ref{line:cor_3}). To update model parameters and weights, Algorithm \ref{alg:adaptive_bgda} performs the descent-ascent scheme, identical to Algorithm \ref{alg:bgda}. Thus, \texttt{AdaptiveBGDA} utilizes \texttt{Adam} \citep{kingma2014adam} and \texttt{RMSProp} \citep{xu2021convergence} to perform descent and ascent steps, respectively. During the experimental study, we tested various combinations of adaptive techniques (e.g., \texttt{Adam}+\texttt{Adam}, \texttt{RMSProp}+\texttt{RMSProp}), but it was \texttt{Adam}+\texttt{RMSProp} that performed the best  (see Appendix \ref{ap:D}).

We utilize \textit{PINNacle} \citep{hao2023pinnacle} as the foundation for our code. To the best of our knowledge, it represents the most recent and comprehensive benchmark, containing a wide variety of PDEs and numerous network architectures. We do not tune competing optimizers but rely on the hyperparameters recommended in this work. 

Although we find a solution to the saddle-point problem \eqref{eq:pinn_saddle}, a neural network $u(\theta)$ relies solely on $\theta$ to make predictions. The weights $\pi$ are required only for equalizing the contribution of all losses in the optimization process. To assess quality, we use \textbf{L2RE}. It is calculated as
\begin{align*}
    \textbf{L2RE}=\sqrt{\frac{\sum_{i=1}^n(y_i-y_i')^2}{\sum_{i=1}^ny_i'^2}},
\end{align*}
where $y_i$ is the output of the model, $y_i'$ is the ground truth, and $n$ is the size of the test set. This metric is generally accepted for numerical methods in mathematical physics. Additionally, it is more sensitive to outliers than \textbf{L1RE}. Since the purpose of this paper is to demonstrate the stability of the proposed approach, we use exactly \textbf{L2RE}.

\subsection{Exploring the Conflicting Gradients}
In Section \ref{sec:intro}, it was noted that one of the main obstacles to obtaining a high-quality model is conflicting gradients. To study this issue, we measure the ratio $\chi=\nicefrac{\|\nabla\mathcal{L}_r(\theta)\|}{\|\nabla\mathcal{L}_b(\theta)\|}$ while training the vanilla \textit{PINN}. We break the iterations into several groups ($I_1=[0,10000)$, $I_2=[10000,20000)$, and $I_3=[20000,30000]$) and examine the distributions of $\chi_1$, $\chi_2$, and $\chi_3$, including their means $\overline{\chi}_1,~\overline{\chi}_2,~\overline{\chi}_3$ and variances $\sigma_1,~\sigma_2,~\sigma_3$. 

In Figure \ref{fig:dominance_adam}, one can see the dynamics of \texttt{MultiAdam} \citep{yao2023multiadam} applied to solve \textit{Poisson 2d-C} \citep{hao2023pinnacle}. It is one of the state-of-the-art optimizers for weighing losses of a \textit{PINN}. From the first epochs, $\|\nabla\mathcal{L}_r(\theta)\|$ demonstrates significant superiority over $\|\nabla\mathcal{L}_b(\theta)\|$. At this stage, we observe $\overline{\chi}_1=3542$, $\sigma_1=7922$. During the next group of iterations, these ratios increase to $\overline{\chi}_2=14996$, $\sigma_2=12025$; and after another $10000$ -- up to $\overline{\chi}_3=16600$, $\sigma_3=11597$. Thus, the value of $\chi$ is rapidly concentrates extremely far away from the desired case of equal magnitudes. This indicates that the $\nabla\mathcal{L}_r(\theta)$ decreases much slower than $\nabla\mathcal{L}_b(\theta)$. Consequently, only $\mathcal{B}$ is well approximated.

The learning process of our method is significantly more stable. Figure \ref{fig:dominance_bgda} shows similar results for the proposed \texttt{AdaptiveBGDA}. Using this scheme, we obtain $\overline{\chi}_1=7$, $\sigma_1=7$; $\overline{\chi}_2=25$, $\sigma_2=27$; $\overline{\chi}_3=45$, $\sigma_3=127$. The problem of conflicting lows still manifests, but it is now much less pronounced. The resulting improvement is statistically significant. Indeed, for $I_1$ only $\approx 31.3\%$ of the values obtained with \texttt{MultiAdam} fall within the $3\sigma_1$-interval for \texttt{AdaptiveBGDA}. At the same time, for $I_2$ and $I_3$ such values do not exist at all.

\begin{figure}[h!] 
   \begin{subfigure}{0.310\textwidth}
           \includegraphics[width=\linewidth]{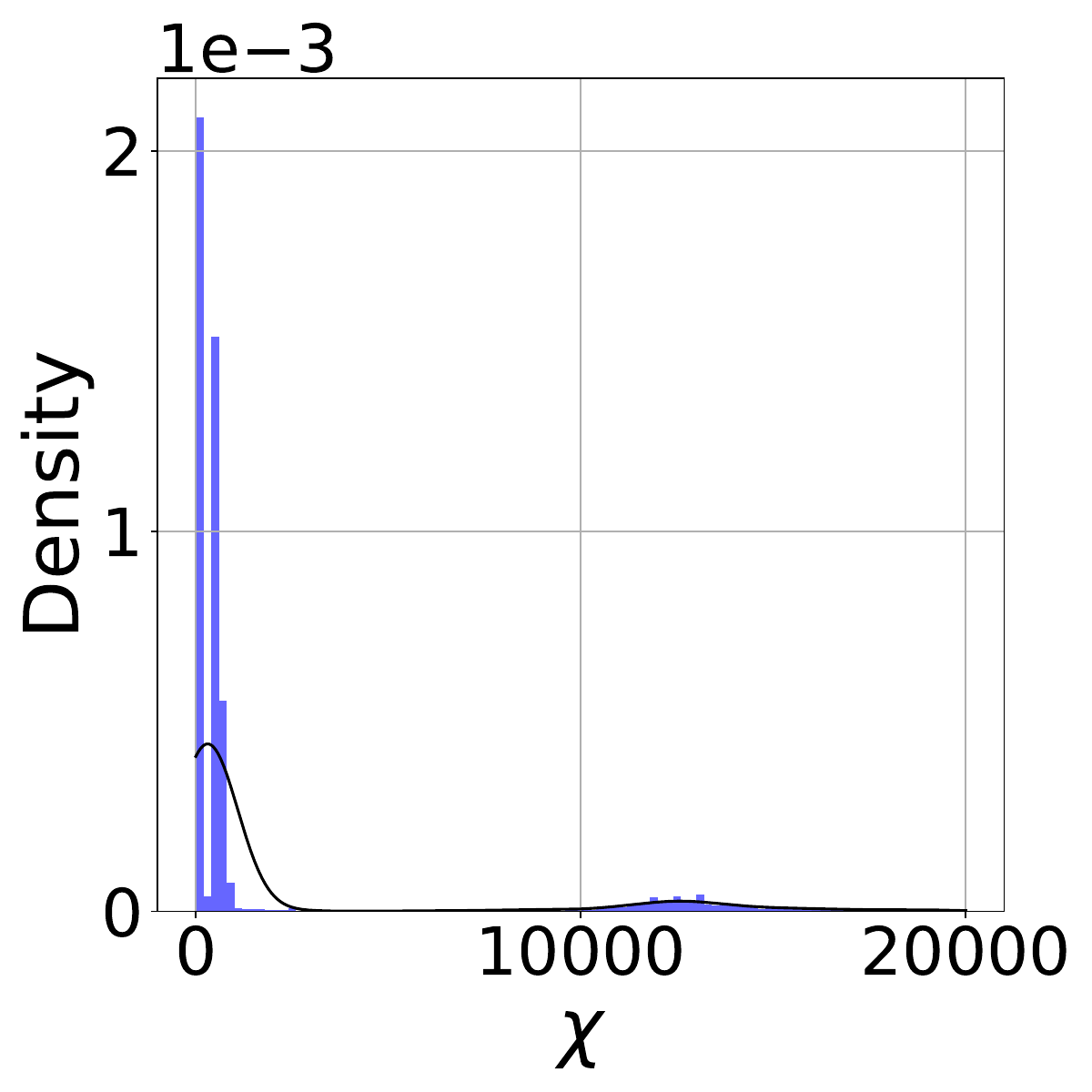}
   \end{subfigure}
   \begin{subfigure}{0.310\textwidth}
       \includegraphics[width=\linewidth]{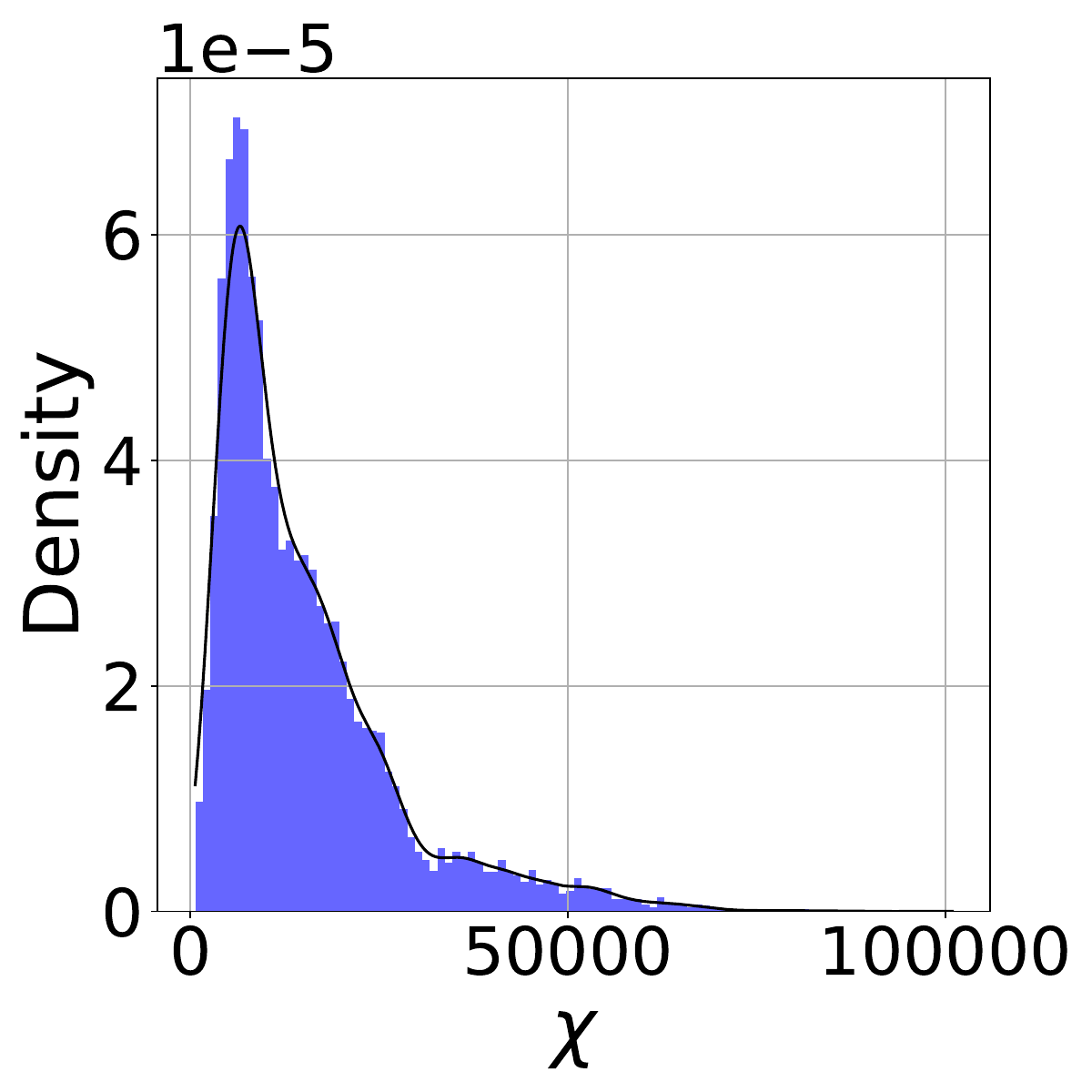}
   \end{subfigure}
   \begin{subfigure}{0.310\textwidth}
       \includegraphics[width=\linewidth]{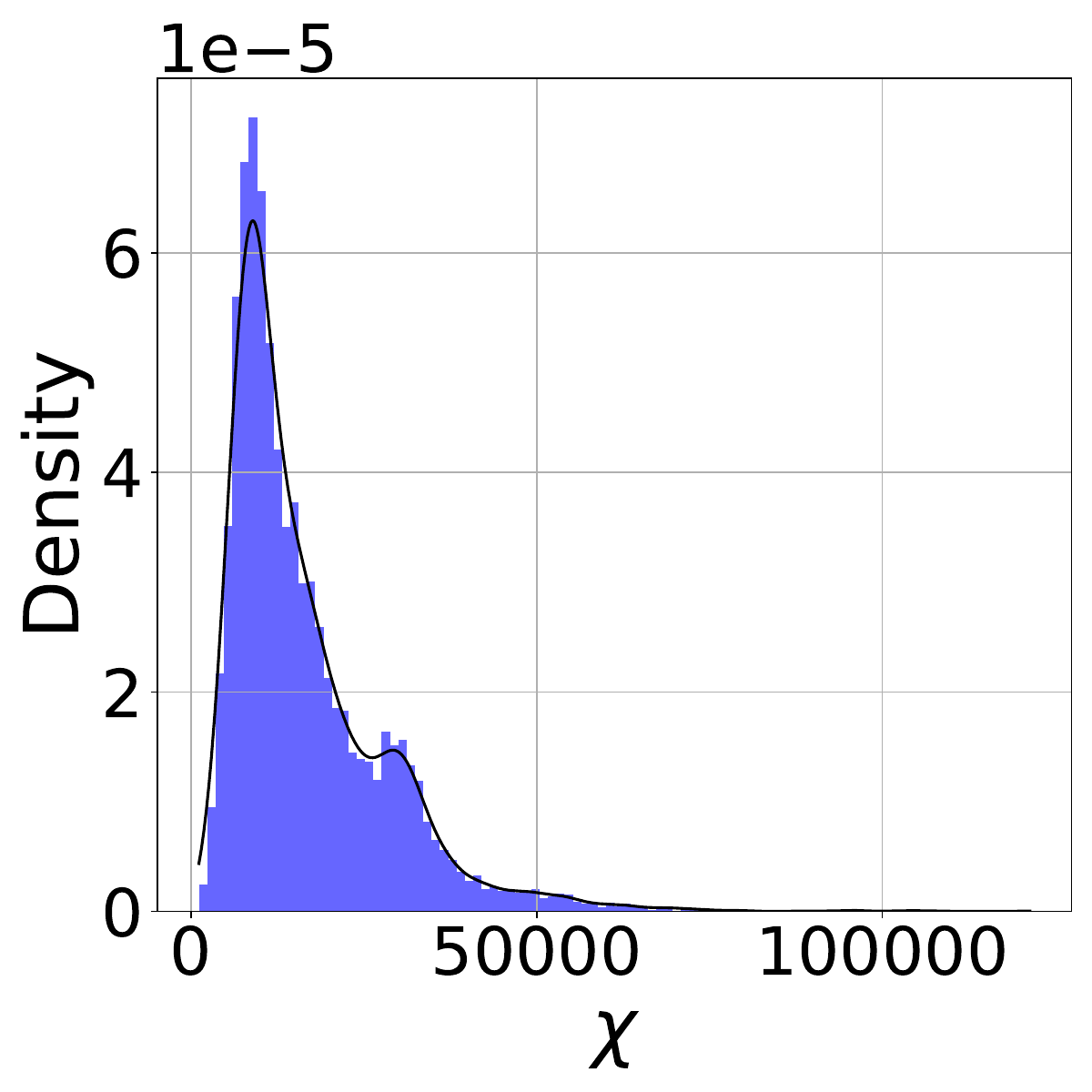}
   \end{subfigure}\\
   \begin{subfigure}{0.310\textwidth}
       \centering
       (a) $0-10000$ epochs
   \end{subfigure}
   \begin{subfigure}{0.310\textwidth}
       \centering
       (b) $10000-20000$ epochs
   \end{subfigure}
   \begin{subfigure}{0.310\textwidth}
       \centering
       (c) $20000-30000$ epochs
   \end{subfigure}
   \caption{Dynamics of $\chi=\nicefrac{\|\nabla\mathcal{L}_r(\theta)\|}{\|\nabla\mathcal{L}_b(\theta)\|}$ during optimization via \texttt{MultiAdam}. The experiment is made on \textit{Poisson 2d-C}. To observe instability, we break the training into three parts}
   \label{fig:dominance_adam}
\end{figure}

\begin{figure}[h!] 
   \begin{subfigure}{0.310\textwidth}
           \includegraphics[width=\linewidth]{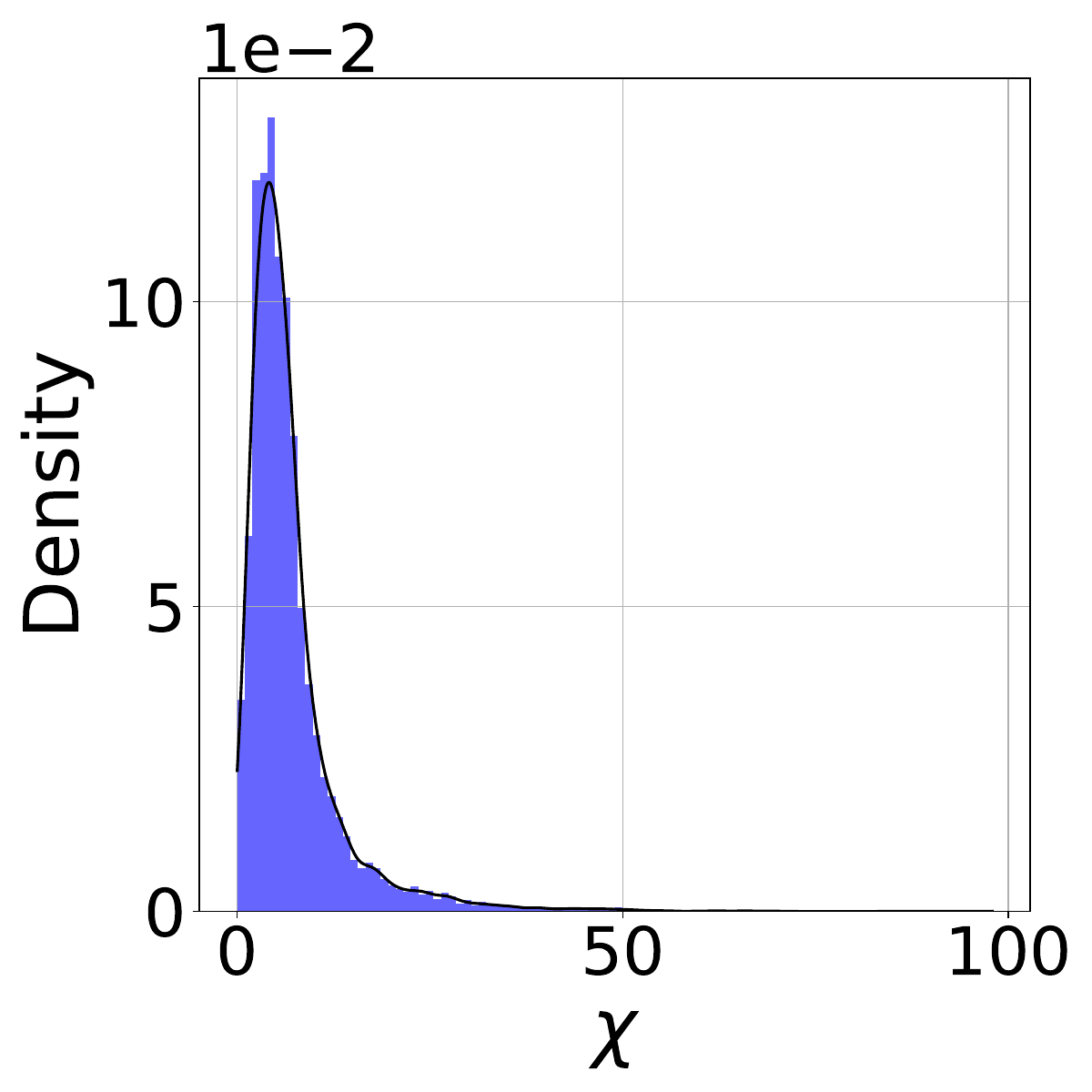}
   \end{subfigure}
   \begin{subfigure}{0.310\textwidth}
       \includegraphics[width=\linewidth]{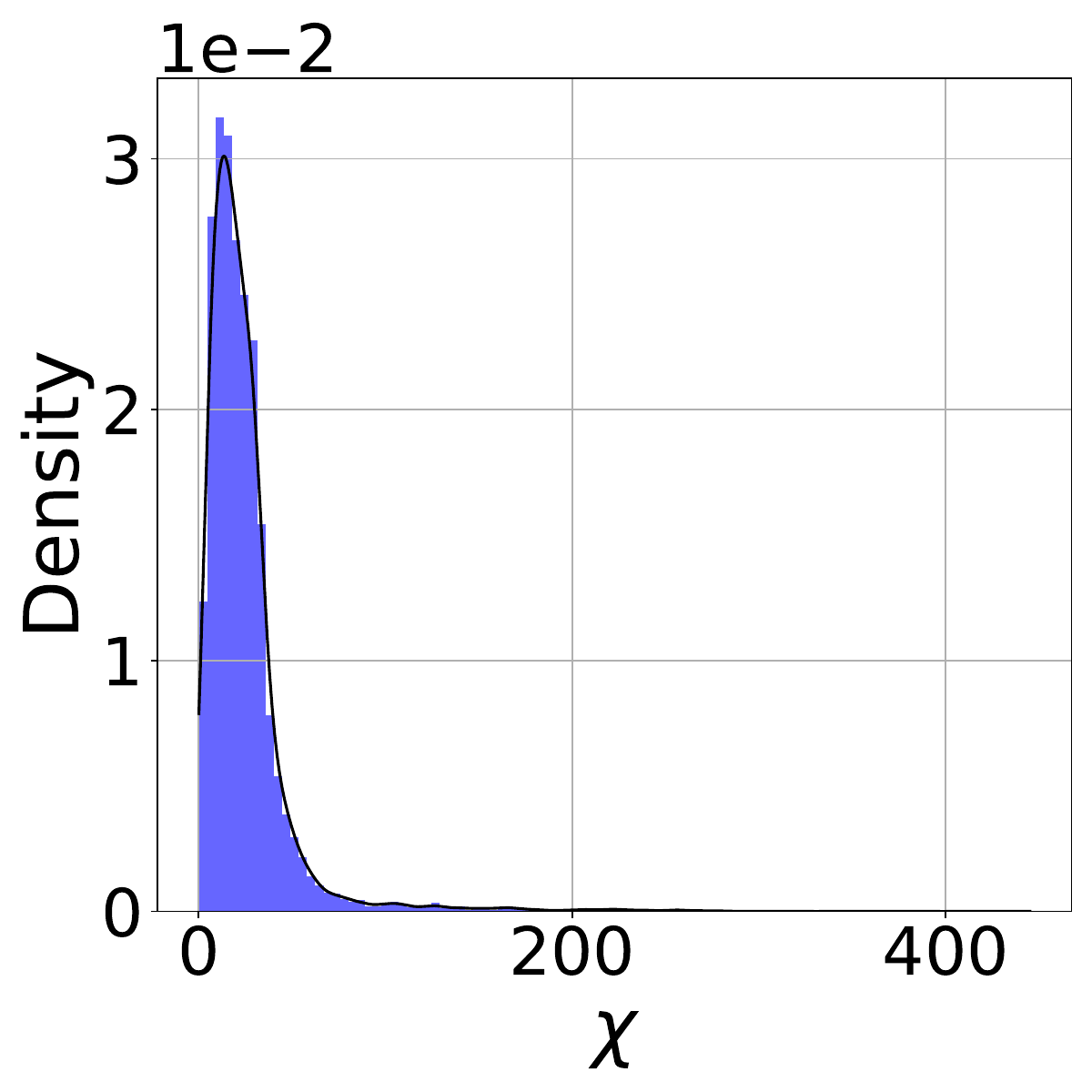}
   \end{subfigure}
   \begin{subfigure}{0.310\textwidth}
       \includegraphics[width=\linewidth]{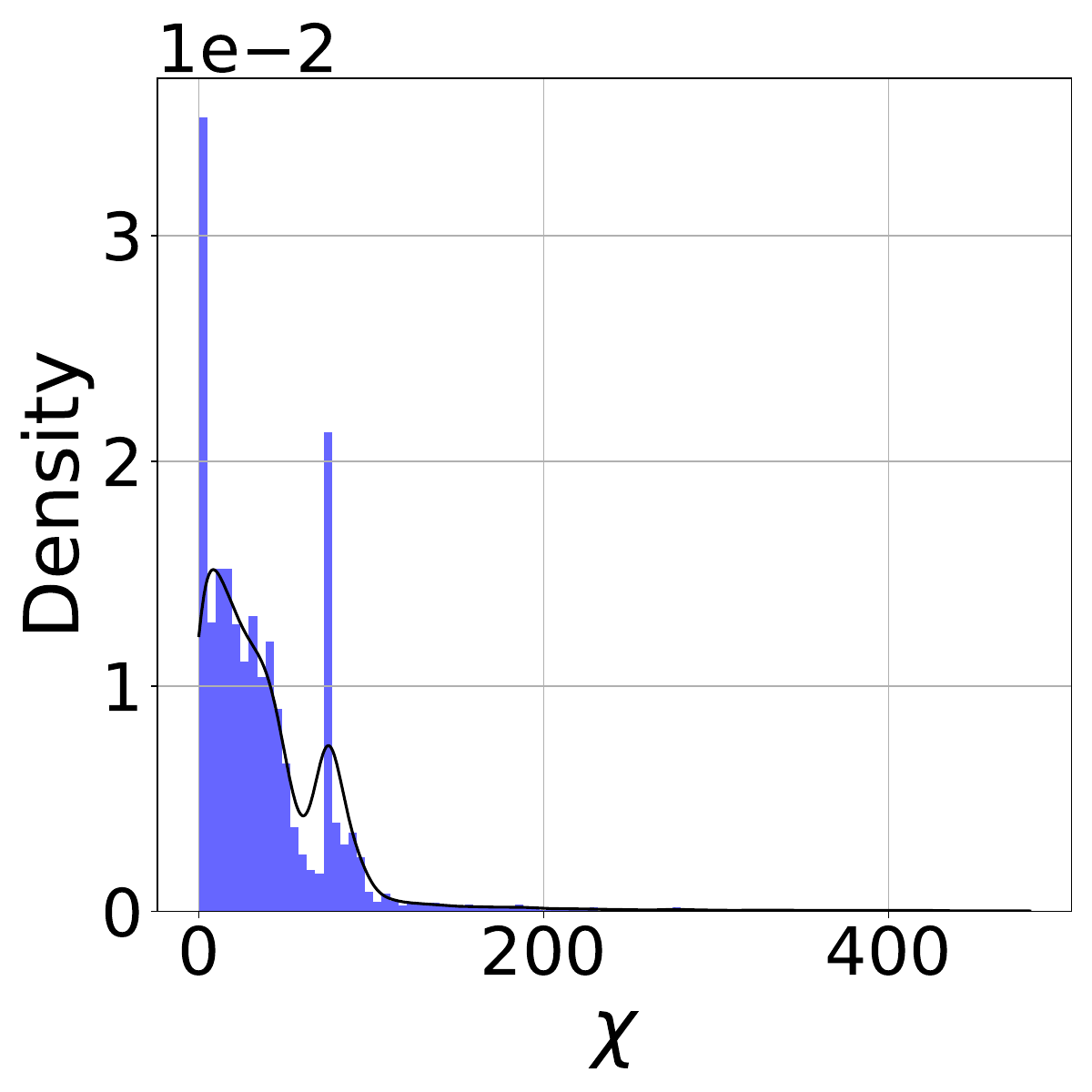}
   \end{subfigure}\\
   \begin{subfigure}{0.310\textwidth}
       \centering
       (a) $0-10000$ epochs
   \end{subfigure}
   \begin{subfigure}{0.310\textwidth}
       \centering
       (b) $10000-20000$ epochs
   \end{subfigure}
   \begin{subfigure}{0.310\textwidth}
       \centering
       (c) $20000-30000$ epochs
   \end{subfigure}
   \caption{Dynamics of $\chi=\nicefrac{\|\nabla\mathcal{L}_r(\theta)\|}{\|\nabla\mathcal{L}_b(\theta)\|}$ during optimization via \texttt{AdaptiveBGDA} \textbf{(our optimizer)}. The experiment is made on \textit{Poisson 2d-C}. To observe instability, we break the training into three parts}
   \label{fig:dominance_bgda}
\end{figure}

\subsection{Exploring Performance of \texttt{AdaptiveBGDA}}
\begin{wrapfigure}[12]{r}{0.6\textwidth}
  \centering
  \vspace{-1.4em}
  \begin{subfigure}[t]{0.48\linewidth}
    \includegraphics[width=\linewidth]{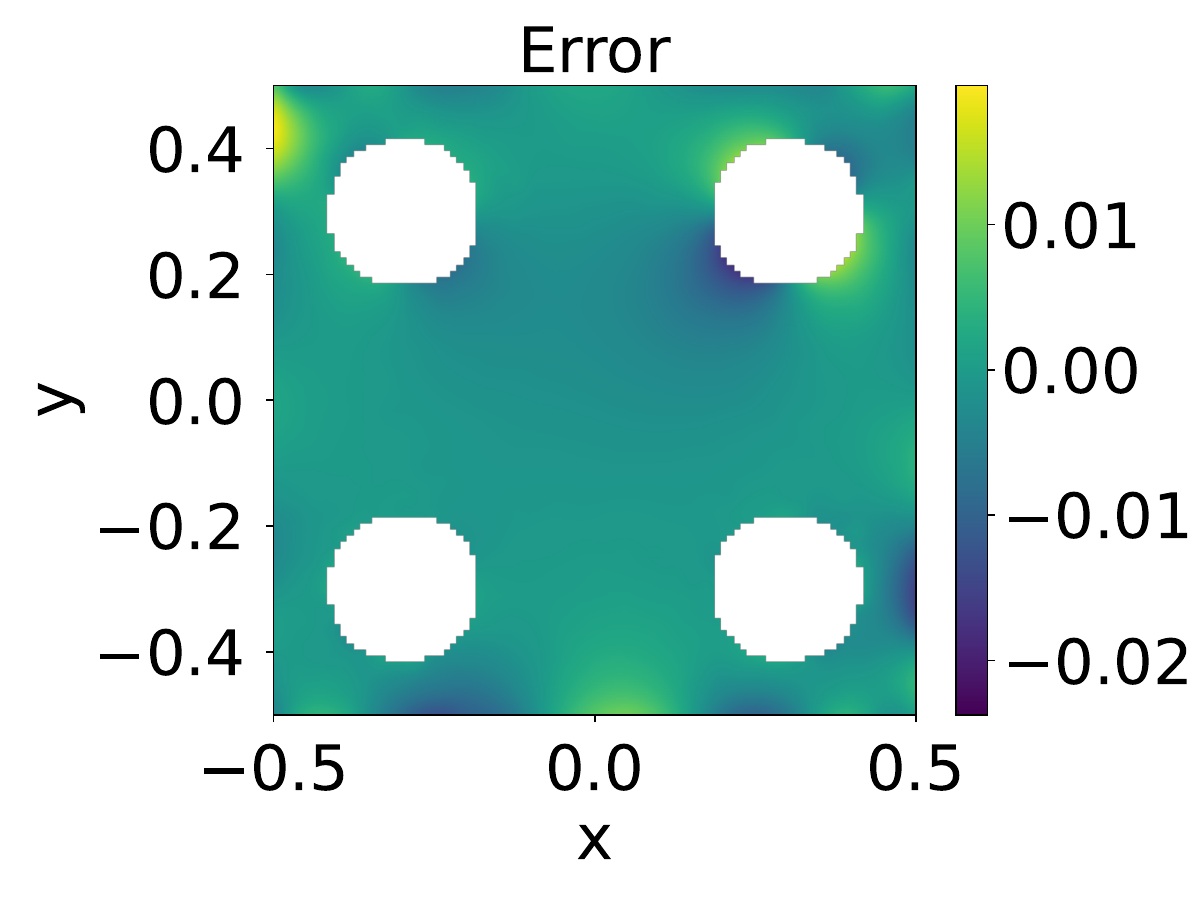}
  \end{subfigure}
  \hspace{0.02\linewidth}
  \begin{subfigure}[t]{0.48\linewidth}
    \includegraphics[width=\linewidth]{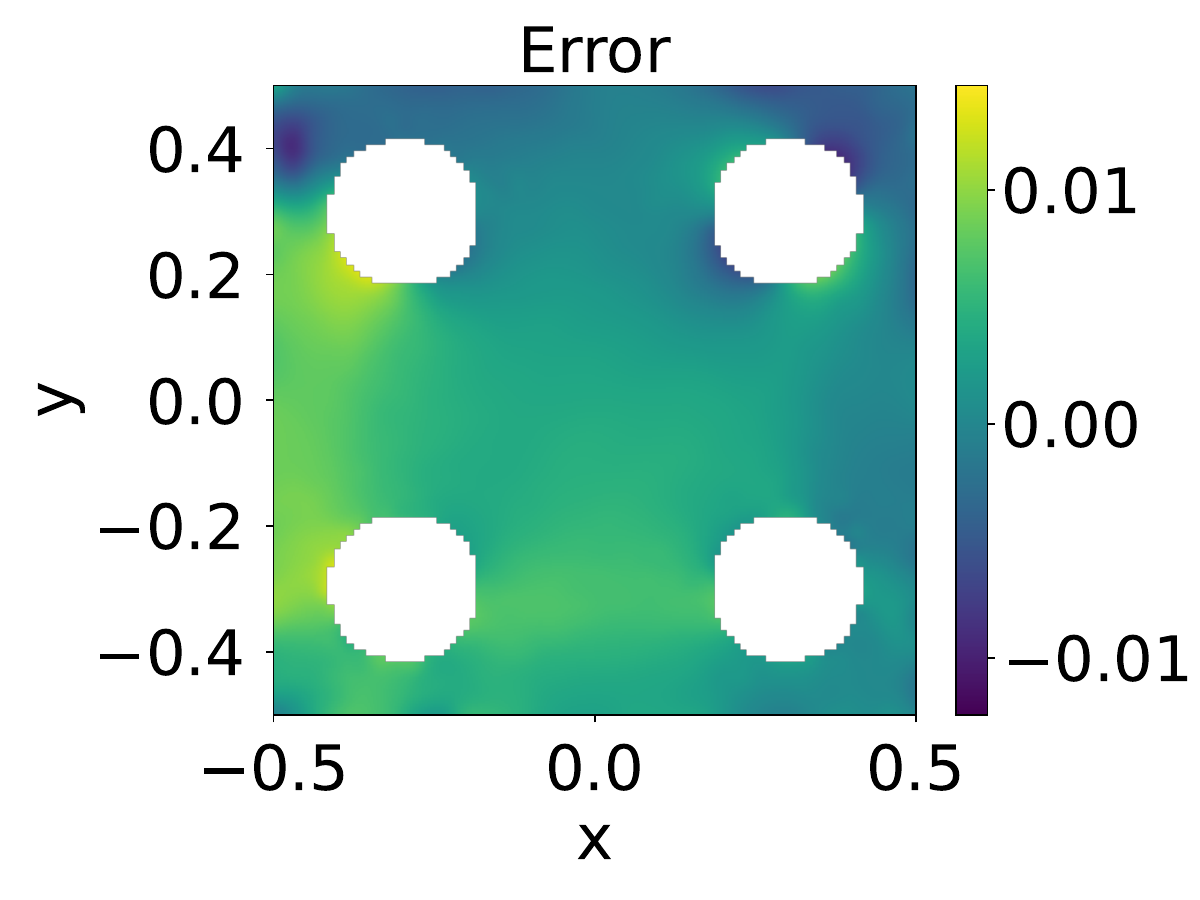}
  \end{subfigure}
  \caption{Heat maps of signed relative errors of \textit{PINN} trained to solve \textit{Poisson 2d-C}. \texttt{AdaptiveBGDA} (left) is compared with \texttt{NTK} (right).}
  \label{fig:wrap-subfigs}
\end{wrapfigure}
To demonstrate the superiority of saddle-point reformulation, we provide an extensive comparison of \texttt{AdaptiveBGDA} (Algorithm \ref{alg:adaptive_bgda}) with existing approaches. We do not adjust the hyperparameters of Algorithm \ref{alg:adaptive_bgda} for each individual PDE to show the potential to outperform the state-of-the-art result. Instead, we fix $\gamma_{\pi}=0.1$, $\gamma_{\theta}=0.008$, $\alpha_1=0.9$, $\alpha_2=0.999$, $\beta=0.999$ and use them consistently throughout the benchmark. To handle the non-convex landscape of $\mathcal{L}(\theta,\pi)$ in $\theta$, we linearly reduce $\gamma_{\theta}$ from the initial value to $0.0004$.

To evaluate the learning potential and generalization capabilities of our approach, we consider $22$ partial differential equations sourced from the \textit{PINNacle} benchmark \citep{hao2023pinnacle}. The collection of tasks covers a broad spectrum of real-world problems. In addition to different operators $\mathcal{R}_i$, $\mathcal{B}_i$, they also differ in the geometry of the $\Omega$ and $\partial\Omega$ regions. Below, we summarize the main features encountered in the selected PDEs.
\begin{itemize}
    \item \textbf{Complex geometry.} Some pieces of the region $\Omega$ are cut out. Since the domain ceases to be simple connected, the solution becomes more complicated, including in terms of numerical retrieval. Problems of this class often arise in applications. For example, the flow of a fluid through an obstacle.
    \item \textbf{Multiple domains.} The region $\Omega$ is divided into several chunks. When moving from one to another, the parameters of PDE change abruptly. The need to perform well for all domains immediately complicates the task.
    \item \textbf{Varying coefficients.} The parameters of PDE vary continuously with the coordinates. Tasks of this type have a role in many applications from heat transfer in materials to population dynamics.
    \item \textbf{Long time.} PDE needs to be solved over a large time interval. This feature is the most difficult for modern architectures and optimizers.
\end{itemize}
For comparison purposes, training of the vanilla \textit{PINN} is considered within this section. Experiments with more complex architectures can be found in Appendix \ref{ap:D}. As competitors, we consider all methods presented in \textit{PINNacle} \citep{hao2023pinnacle}: \texttt{LBFGS} \citep{byrd1995limited}, \texttt{Adam} \citep{kingma2014adam}, \texttt{MultiAdam} \citep{yao2023multiadam}, and combinations of \texttt{Adam} with \texttt{RAR} \citep{lu2021deepxde}, \texttt{LRA} \citep{wang2021understanding}, \texttt{NTK} \citep{wang2022and}.
\definecolor{best}{RGB}{0, 0, 139}
\definecolor{second}{RGB}{100, 149, 237}
\begin{table}[htbp]
\centering
\scriptsize
\renewcommand{\arraystretch}{1.2}
\setlength{\tabcolsep}{4pt}
\caption{Comparison of \texttt{AdaptiveBGDA} to the existing techniques. In all experiments, the model is trained to the performance limit. \textbf{L2RE} is used as a quality metric. We highlight the \protect\colorbox{myblue!20}{\textbf{best}} and the \protect\colorbox{mylightblue!20}{second best} results for each PDE.}
\begin{tabular}{llccccccc}
\toprule
\textbf{PDE} & & \textbf{Adam} & \textbf{LBFGS} & \textbf{LRA} & \textbf{NTK} & \textbf{RAR} & \textbf{MultiAdam} & \textbf{BGDA (this paper)} \\
\midrule

\multirow{2}{*}{Burgers}
 & 1d-C & 1.45E-2 & \cellcolor{second!20}1.33E-2 & 2.61E-2 & 1.84E-2 & 3.32E-2 & 4.85E-2 & \cellcolor{best!20}\textbf{1.30E-2} \\
 & 2d-C & \cellcolor{second!20}2.70E-1 & 4.65E-1 & \cellcolor{best!20}\textbf{2.60E-1} & 2.75E-1 & 3.45E-1 & 3.33E-1 & 4.21E-1 \\
 
\midrule
\multirow{4}{*}{Poisson}
 & 2d-C & 3.49E-2 & NaN & 1.17E-1 & \cellcolor{second!20}1.23E-2 & 6.99E-1 & 2.63E-2 & \cellcolor{best!20}\textbf{8.16E-3} \\
 & 2d-CG & 6.08E-2 & 2.96E-1 & 4.34E-2 & \cellcolor{best!20}\textbf{1.43E-2} & 6.48E-1 & 2.76E-1 & \cellcolor{second!20}1.76E-2 \\
 & 3d-CG & 3.74E-1 & 7.05E-1 & \cellcolor{second!20}1.02E-1 & 9.47E-1 & 5.76E-1 & 3.63E-1 & \cellcolor{best!20}\textbf{4.78E-2} \\
 & 2d-MS & 6.30E-1 & 1.45E+0 & 7.94E-1 & 7.48E-1 & 6.44E-1 & \cellcolor{second!20}5.90E-1 & \cellcolor{best!20}\textbf{3.48E-1} \\
 
\midrule
\multirow{4}{*}{Heat}
 & 2d-VC & 2.35E-1 & 2.32E-1 & \cellcolor{best!20}\textbf{2.12E-1} & \cellcolor{second!20}2.14E-1 & 9.66E-1 & 4.75E-1 & 2.93E-1 \\
 & 2d-MS & 6.21E-2 & \cellcolor{best!20}\textbf{1.73E-2} & 8.79E-2 & 4.40E-2 & 7.49E-2 & 2.18E-1 & \cellcolor{second!20}1.88E-2 \\
 & 2d-CG & 3.64E-2 & 8.57E-1 & 1.25E-1 & 1.16E-1 & \cellcolor{second!20}2.72E-2 & 7.12E-2 & \cellcolor{best!20}\textbf{1.01E-2} \\
 & 2d-LT & 9.99E-1 & 1.00E+0 & 9.99E-1 & 1.00E+0 & 9.99E-1 & 1.00E+0 & \cellcolor{best!20}\textbf{9.98E-1} \\
 
\midrule
\multirow{3}{*}{NS}
 & 2d-C & \cellcolor{second!20}4.70E-2 & 2.14E-1 & NaN & 1.98E-1 & 4.69E-1 & 7.27E-1 & \cellcolor{best!20}\textbf{2.24E-2} \\
 & 2d-CG & \cellcolor{second!20}1.19E-1 & NaN & 3.32E-1 & 2.93E-1 & 3.34E-1 & 4.31E-1 & \cellcolor{best!20}\textbf{7.63E-2} \\
 & 2d-LT & 9.96E-1 & \cellcolor{best!20}\textbf{9.70E-1} & 1.00E+0 & 9.99E-1 & 1.00E+0 & 1.00E+0 & \cellcolor{second!20}9.75E-1 \\

\midrule
\multirow{3}{*}{Wave}
 & 1d-C & 2.85E-1 & NaN & 3.61E-1 & \cellcolor{second!20}9.79E-2 & 5.39E-1 & 1.21E-1 & \cellcolor{best!20}\textbf{1.62E-2} \\
 & 2d-CG & 1.66E+0 & 1.33E+0 & 1.48E+0 & 2.16E+0 & 1.15E+0 & \cellcolor{second!20}1.09E+0 & \cellcolor{best!20}\textbf{7.78E-1} \\
 & 2d-MS & 1.02E+0 & 1.37E+0 & 1.02E+0 & 1.04E+0 & 1.35E+0 & \cellcolor{second!20}1.01E+0 & \cellcolor{best!20}\textbf{8.98E-1} \\
 
\midrule
\multirow{2}{*}{Chaotic}
 & GS & 1.58E-1 & NaN & \cellcolor{second!20}9.37E-2 & 2.16E-1 & 9.46E-2 & \cellcolor{second!20}9.37E-2 & \cellcolor{best!20}\textbf{9.30E-2} \\
 & KS & 9.86E-1 & NaN & \cellcolor{second!20}9.57E-1 & 9.64E-1 & 1.01E+0 & 9.61E-1 & \cellcolor{best!20}\textbf{9.53E-1} \\
 
\midrule
\multirow{2}{*}{High dim}
 & PNd & 2.58E-3 & 4.67E-4 & \cellcolor{second!20}4.58E-4 & 4.64E-3 & 3.59E-3 & 3.98E-3 & \cellcolor{best!20}\textbf{1.20E-4} \\
 & HNd & 3.61E-1 & \cellcolor{best!20}\textbf{1.19E-4} & 3.94E-1 & 3.97E-1 & 3.57E-1 & 3.02E-1 & \cellcolor{second!20}1.60E-4 \\
 
\midrule
\multirow{2}{*}{Inverse}
 & PInv & 9.42E-2 & NaN & 1.54E-1 & 1.93E-1 & \cellcolor{second!20}9.35E-2 & 1.30E-1 & \cellcolor{best!20}\textbf{8.59E-2} \\
 & HInv & 5.26E-2 & NaN & \cellcolor{second!20}5.09E-2 & 7.52E-2 & 1.52E+0 & 8.04E-2 & \cellcolor{best!20}\textbf{4.05E-2} \\

\bottomrule
\end{tabular}
\label{tab:comparison}
\end{table}

It can be seen from Table \ref{tab:comparison} that \texttt{AdaptiveBGDA} (Algorithm \ref{alg:adaptive_bgda}) is dominant in $72.7\%$ cases. The previous record of $22.7\%$ belonged to \texttt{LRA}. In $18.2\%$ cases, the quality is improved by more than double. The superiority of our method is particularly well demonstrated by the error heat maps. Such a comparison is presented in Figure \ref{fig:wrap-subfigs}. In the right part of Figure \ref{fig:wrap-subfigs}, we observe a significant region within the interior of the domain where the approximated solution exhibits a large error. The absence of such a region on the left side of Figure \ref{fig:wrap-subfigs} illustrates that we successfully address the issue of underestimating losses in the interior of the domain. Below, we analyze the performance of our approach in the conducted experiments:
\begin{itemize}
    \item Standard PDEs without special features (\textit{Burgers 1d-C}, \textit{Burgers 2d-C}, \textit{NS 2d-C}, \textit{Wave 1d-C}) have a simpler loss landscape in $\theta$. Nevertheless, \texttt{AdaptiveBGDA} gives a noticeable improvement when solving tasks from this class. 
    \item One of the strongest quality gains is observed on problems with multiple subdomains (\textit{Poisson 3d-CG}, \textit{Poisson 2d-MS}). Even without fine-tuning, our approach turns out to be good enough to adapt to them. Indeed, the saddle-point setting is introduced in order to fairly account for the contribution of operators to losses. The expected result is a better adaptation to all subdomains simultaneously.
    \item For problems with complex geometry (\textit{Poisson 2d-C}, \textit{Poisson 2d-CG}, \textit{Poisson 3d-CG}, \textit{Heat 2d-CG}, \textit{NS 2d-CG}, \textit{Wave 2d-CG}), an improvement is also observed in five out of six cases.
    \item Unexpectedly, \texttt{AdaptiveBGDA} shows quality gains in exotic settings such as \textit{Chaotic} or \textit{Inverse}.
\end{itemize}
\section{Discussion}
In this paper, we note that even advanced weighting schemes for \textit{PINN} do not achieve a fully balanced optimization process. To address this issue, we reformulate the training problem as the nonconvex-strongly concave SPP of non-Euclidean nature. In addition to theoretical analysis, we conduct a comprehensive empirical study. We observe a significant increase in model quality (Table \ref{tab:comparison}). We also note an increase in the stability of the optimization process (Figure \ref{fig:dominance_bgda}). Specifically, the losses within the domain decrease approximately as rapidly as those at the boundary. This effect is noticeable in practice (Figure \ref{fig:wrap-subfigs}).
\bibliographystyle{plainnat}
\bibliography{main}

\begin{thebibliography}{73}
\providecommand{\natexlab}[1]{#1}
\providecommand{\url}[1]{\texttt{#1}}
\expandafter\ifx\csname urlstyle\endcsname\relax
  \providecommand{\doi}[1]{doi: #1}\else
  \providecommand{\doi}{doi: \begingroup \urlstyle{rm}\Url}\fi

\bibitem[Adolphs et~al.(2019)Adolphs, Daneshmand, Lucchi, and Hofmann]{adolphs2019local}
Leonard Adolphs, Hadi Daneshmand, Aurelien Lucchi, and Thomas Hofmann.
\newblock Local saddle point optimization: A curvature exploitation approach.
\newblock In \emph{The 22nd International Conference on Artificial Intelligence and Statistics}, pages 486--495. PMLR, 2019.

\bibitem[Anagnostopoulos et~al.(2024)Anagnostopoulos, Toscano, Stergiopulos, and Karniadakis]{anagnostopoulos2024residual}
Sokratis~J Anagnostopoulos, Juan~Diego Toscano, Nikolaos Stergiopulos, and George~Em Karniadakis.
\newblock Residual-based attention in physics-informed neural networks.
\newblock \emph{Computer Methods in Applied Mechanics and Engineering}, 421:\penalty0 116805, 2024.

\bibitem[Arnold et~al.(2002)Arnold, Brezzi, Cockburn, and Marini]{arnold2002unified}
Douglas~N Arnold, Franco Brezzi, Bernardo Cockburn, and L~Donatella Marini.
\newblock Unified analysis of discontinuous galerkin methods for elliptic problems.
\newblock \emph{SIAM journal on numerical analysis}, 39\penalty0 (5):\penalty0 1749--1779, 2002.

\bibitem[Babu{\v{s}}ka and Aziz(1976)]{babuvska1976angle}
Ivo Babu{\v{s}}ka and A~Kadir Aziz.
\newblock On the angle condition in the finite element method.
\newblock \emph{SIAM Journal on numerical analysis}, 13\penalty0 (2):\penalty0 214--226, 1976.

\bibitem[Bateman(1932)]{bateman1932partial}
Harry Bateman.
\newblock Partial differential equations of mathematical physics.
\newblock \emph{Partial Differential Equations of Mathematical Physics}, 1932.

\bibitem[Beznosikov et~al.(2023)Beznosikov, Gasnikov, Zainullina, Maslovskii, and Pasechnyuk]{beznosikov2023unified}
Aleksandr~Nikolaevich Beznosikov, Alexander~Vladimirovich Gasnikov, Karina~E Zainullina, A~Yu Maslovskii, and Dmitry~Arkad'evich Pasechnyuk.
\newblock A unified analysis of variational inequality methods: variance reduction, sampling, quantization, and coordinate descent.
\newblock \emph{Computational Mathematics and Mathematical Physics}, 63\penalty0 (2):\penalty0 147--174, 2023.

\bibitem[Bischof and Kraus(2025)]{bischof2025multi}
Rafael Bischof and Michael~A Kraus.
\newblock Multi-objective loss balancing for physics-informed deep learning.
\newblock \emph{Computer Methods in Applied Mechanics and Engineering}, 439:\penalty0 117914, 2025.

\bibitem[Bonfanti et~al.(2024)Bonfanti, Bruno, and Cipriani]{bonfanti2024challenges}
Andrea Bonfanti, Giuseppe Bruno, and Cristina Cipriani.
\newblock The challenges of the nonlinear regime for physics-informed neural networks.
\newblock \emph{Advances in Neural Information Processing Systems}, 37:\penalty0 41852--41881, 2024.

\bibitem[Boroun et~al.(2023)Boroun, Yazdandoost~Hamedani, and Jalilzadeh]{boroun2023projection}
Morteza Boroun, Erfan Yazdandoost~Hamedani, and Afrooz Jalilzadeh.
\newblock Projection-free methods for solving nonconvex-concave saddle point problems.
\newblock \emph{Advances in Neural Information Processing Systems}, 36:\penalty0 53844--53856, 2023.

\bibitem[Byrd et~al.(1995)Byrd, Lu, Nocedal, and Zhu]{byrd1995limited}
Richard~H Byrd, Peihuang Lu, Jorge Nocedal, and Ciyou Zhu.
\newblock A limited memory algorithm for bound constrained optimization.
\newblock \emph{SIAM Journal on scientific computing}, 16\penalty0 (5):\penalty0 1190--1208, 1995.

\bibitem[Chen et~al.(2018)Chen, Badrinarayanan, Lee, and Rabinovich]{chen2018gradnorm}
Zhao Chen, Vijay Badrinarayanan, Chen-Yu Lee, and Andrew Rabinovich.
\newblock Gradnorm: Gradient normalization for adaptive loss balancing in deep multitask networks.
\newblock In \emph{International conference on machine learning}, pages 794--803. PMLR, 2018.

\bibitem[Cheridito et~al.(2021)Cheridito, Jentzen, and Rossmannek]{cheridito2021efficient}
Patrick Cheridito, Arnulf Jentzen, and Florian Rossmannek.
\newblock Efficient approximation of high-dimensional functions with neural networks.
\newblock \emph{IEEE Transactions on Neural Networks and Learning Systems}, 33\penalty0 (7):\penalty0 3079--3093, 2021.

\bibitem[Cho et~al.(2023)Cho, Lee, Rim, and Park]{cho2023hypernetwork}
Woojin Cho, Kookjin Lee, Donsub Rim, and Noseong Park.
\newblock Hypernetwork-based meta-learning for low-rank physics-informed neural networks.
\newblock \emph{Advances in Neural Information Processing Systems}, 36:\penalty0 11219--11231, 2023.

\bibitem[Courant et~al.(1967)Courant, Friedrichs, and Lewy]{courant1967partial}
Richard Courant, Kurt Friedrichs, and Hans Lewy.
\newblock On the partial difference equations of mathematical physics.
\newblock \emph{IBM journal of Research and Development}, 11\penalty0 (2):\penalty0 215--234, 1967.

\bibitem[Courant et~al.(1994)]{courant1994variational}
Richard Courant et~al.
\newblock Variational methods for the solution of problems of equilibrium and vibrations.
\newblock \emph{Lecture notes in pure and applied mathematics}, pages 1--1, 1994.

\bibitem[Cybenko(1989)]{cybenko1989approximation}
George Cybenko.
\newblock Approximation by superpositions of a sigmoidal function.
\newblock \emph{Mathematics of control, signals and systems}, 2\penalty0 (4):\penalty0 303--314, 1989.

\bibitem[Dissanayake and Phan-Thien(1994)]{dissanayake1994neural}
MWM~Gamini Dissanayake and Nhan Phan-Thien.
\newblock Neural-network-based approximations for solving partial differential equations.
\newblock \emph{communications in Numerical Methods in Engineering}, 10\penalty0 (3):\penalty0 195--201, 1994.

\bibitem[Du and Hu(2019)]{du2019linear}
Simon~S Du and Wei Hu.
\newblock Linear convergence of the primal-dual gradient method for convex-concave saddle point problems without strong convexity.
\newblock In \emph{The 22nd International Conference on Artificial Intelligence and Statistics}, pages 196--205. PMLR, 2019.

\bibitem[D{\"u}ster et~al.(2008)D{\"u}ster, Parvizian, Yang, and Rank]{duster2008finite}
Alexander D{\"u}ster, Jamshid Parvizian, Zhengxiong Yang, and Ernst Rank.
\newblock The finite cell method for three-dimensional problems of solid mechanics.
\newblock \emph{Computer methods in applied mechanics and engineering}, 197\penalty0 (45-48):\penalty0 3768--3782, 2008.

\bibitem[Evans(2022)]{evans2022partial}
Lawrence~C Evans.
\newblock \emph{Partial differential equations}, volume~19.
\newblock American Mathematical Society, 2022.

\bibitem[Guo et~al.(2016)Guo, Li, and Iorio]{guo2016convolutional}
Xiaoxiao Guo, Wei Li, and Francesco Iorio.
\newblock Convolutional neural networks for steady flow approximation.
\newblock In \emph{Proceedings of the 22nd ACM SIGKDD international conference on knowledge discovery and data mining}, pages 481--490, 2016.

\bibitem[Hao et~al.(2023)Hao, Yao, Su, Su, Wang, Lu, Xia, Zhang, Liu, Lu, et~al.]{hao2023pinnacle}
Zhongkai Hao, Jiachen Yao, Chang Su, Hang Su, Ziao Wang, Fanzhi Lu, Zeyu Xia, Yichi Zhang, Songming Liu, Lu~Lu, et~al.
\newblock Pinnacle: A comprehensive benchmark of physics-informed neural networks for solving pdes.
\newblock \emph{arXiv preprint arXiv:2306.08827}, 2023.

\bibitem[Heusel et~al.(2017)Heusel, Ramsauer, Unterthiner, Nessler, and Hochreiter]{heusel2017gans}
Martin Heusel, Hubert Ramsauer, Thomas Unterthiner, Bernhard Nessler, and Sepp Hochreiter.
\newblock Gans trained by a two time-scale update rule converge to a local nash equilibrium.
\newblock \emph{Advances in neural information processing systems}, 30, 2017.

\bibitem[Heydari et~al.(2019)Heydari, Thompson, and Mehmood]{heydari2019softadapt}
A~Ali Heydari, Craig~A Thompson, and Asif Mehmood.
\newblock Softadapt: Techniques for adaptive loss weighting of neural networks with multi-part loss functions.
\newblock \emph{arXiv preprint arXiv:1912.12355}, 2019.

\bibitem[Hou et~al.(2023)Hou, Li, and Ying]{hou2023enhancing}
Jie Hou, Ying Li, and Shihui Ying.
\newblock Enhancing pinns for solving pdes via adaptive collocation point movement and adaptive loss weighting.
\newblock \emph{Nonlinear Dynamics}, 111\penalty0 (16):\penalty0 15233--15261, 2023.

\bibitem[Huang et~al.(2021)Huang, Wu, and Huang]{huang2021efficient}
Feihu Huang, Xidong Wu, and Heng Huang.
\newblock Efficient mirror descent ascent methods for nonsmooth minimax problems.
\newblock \emph{Advances in Neural Information Processing Systems}, 34:\penalty0 10431--10443, 2021.

\bibitem[Hwang and Lim(2024)]{hwang2024dual}
Youngsik Hwang and Dongyoung Lim.
\newblock Dual cone gradient descent for training physics-informed neural networks.
\newblock \emph{Advances in Neural Information Processing Systems}, 37:\penalty0 98563--98595, 2024.

\bibitem[Jaggi(2013)]{jaggi2013revisiting}
Martin Jaggi.
\newblock Revisiting frank-wolfe: Projection-free sparse convex optimization.
\newblock In \emph{International conference on machine learning}, pages 427--435. PMLR, 2013.

\bibitem[Jagtap et~al.(2020{\natexlab{a}})Jagtap, Kawaguchi, and Em~Karniadakis]{jagtap2020locally}
Ameya~D Jagtap, Kenji Kawaguchi, and George Em~Karniadakis.
\newblock Locally adaptive activation functions with slope recovery for deep and physics-informed neural networks.
\newblock \emph{Proceedings of the Royal Society A}, 476\penalty0 (2239):\penalty0 20200334, 2020{\natexlab{a}}.

\bibitem[Jagtap et~al.(2020{\natexlab{b}})Jagtap, Kawaguchi, and Karniadakis]{jagtap2020adaptive}
Ameya~D Jagtap, Kenji Kawaguchi, and George~Em Karniadakis.
\newblock Adaptive activation functions accelerate convergence in deep and physics-informed neural networks.
\newblock \emph{Journal of Computational Physics}, 404:\penalty0 109136, 2020{\natexlab{b}}.

\bibitem[Jin et~al.(2019)Jin, Netrapalli, and Jordan]{jin2019minmax}
Chi Jin, Praneeth Netrapalli, and Michael~I Jordan.
\newblock Minmax optimization: Stable limit points of gradient descent ascent are locally optimal.
\newblock \emph{arXiv preprint arXiv:1902.00618}, 2019.

\bibitem[Jin et~al.(2021)Jin, Cai, Li, and Karniadakis]{jin2021nsfnets}
Xiaowei Jin, Shengze Cai, Hui Li, and George~Em Karniadakis.
\newblock Nsfnets (navier-stokes flow nets): Physics-informed neural networks for the incompressible navier-stokes equations.
\newblock \emph{Journal of Computational Physics}, 426:\penalty0 109951, 2021.

\bibitem[Kendall et~al.(2018)Kendall, Gal, and Cipolla]{kendall2018multi}
Alex Kendall, Yarin Gal, and Roberto Cipolla.
\newblock Multi-task learning using uncertainty to weigh losses for scene geometry and semantics.
\newblock In \emph{Proceedings of the IEEE conference on computer vision and pattern recognition}, pages 7482--7491, 2018.

\bibitem[Kingma and Ba(2014)]{kingma2014adam}
Diederik~P Kingma and Jimmy Ba.
\newblock Adam: A method for stochastic optimization.
\newblock \emph{arXiv preprint arXiv:1412.6980}, 2014.

\bibitem[Kong and Monteiro(2021)]{kong2021accelerated}
Weiwei Kong and Renato~DC Monteiro.
\newblock An accelerated inexact proximal point method for solving nonconvex-concave min-max problems.
\newblock \emph{SIAM Journal on Optimization}, 31\penalty0 (4):\penalty0 2558--2585, 2021.

\bibitem[Korpelevich(1976)]{korpelevich1976extragradient}
Galina~M Korpelevich.
\newblock The extragradient method for finding saddle points and other problems.
\newblock \emph{Matecon}, 12:\penalty0 747--756, 1976.

\bibitem[Krishnapriyan et~al.(2021)Krishnapriyan, Gholami, Zhe, Kirby, and Mahoney]{krishnapriyan2021characterizing}
Aditi Krishnapriyan, Amir Gholami, Shandian Zhe, Robert Kirby, and Michael~W Mahoney.
\newblock Characterizing possible failure modes in physics-informed neural networks.
\newblock \emph{Advances in neural information processing systems}, 34:\penalty0 26548--26560, 2021.

\bibitem[Lagaris et~al.(1998)Lagaris, Likas, and Fotiadis]{lagaris1998artificial}
Isaac~E Lagaris, Aristidis Likas, and Dimitrios~I Fotiadis.
\newblock Artificial neural networks for solving ordinary and partial differential equations.
\newblock \emph{IEEE transactions on neural networks}, 9\penalty0 (5):\penalty0 987--1000, 1998.

\bibitem[Li et~al.(2020)Li, Kovachki, Azizzadenesheli, Liu, Bhattacharya, Stuart, and Anandkumar]{li2020fourier}
Zongyi Li, Nikola Kovachki, Kamyar Azizzadenesheli, Burigede Liu, Kaushik Bhattacharya, Andrew Stuart, and Anima Anandkumar.
\newblock Fourier neural operator for parametric partial differential equations.
\newblock \emph{arXiv preprint arXiv:2010.08895}, 2020.

\bibitem[Lin et~al.(2020)Lin, Jin, and Jordan]{lin2020gradient}
Tianyi Lin, Chi Jin, and Michael Jordan.
\newblock On gradient descent ascent for nonconvex-concave minimax problems.
\newblock In \emph{International conference on machine learning}, pages 6083--6093. PMLR, 2020.

\bibitem[Liu and Wang(2021)]{liu2021dual}
Dehao Liu and Yan Wang.
\newblock A dual-dimer method for training physics-constrained neural networks with minimax architecture.
\newblock \emph{Neural Networks}, 136:\penalty0 112--125, 2021.

\bibitem[Liu et~al.(2019)Liu, Johns, and Davison]{liu2019end}
Shikun Liu, Edward Johns, and Andrew~J Davison.
\newblock End-to-end multi-task learning with attention.
\newblock In \emph{Proceedings of the IEEE/CVF conference on computer vision and pattern recognition}, pages 1871--1880, 2019.

\bibitem[Lu et~al.(2018)Lu, Freund, and Nesterov]{lu2018relatively}
Haihao Lu, Robert~M Freund, and Yurii Nesterov.
\newblock Relatively smooth convex optimization by first-order methods, and applications.
\newblock \emph{SIAM Journal on Optimization}, 28\penalty0 (1):\penalty0 333--354, 2018.

\bibitem[Lu et~al.(2021)Lu, Meng, Mao, and Karniadakis]{lu2021deepxde}
Lu~Lu, Xuhui Meng, Zhiping Mao, and George~Em Karniadakis.
\newblock Deepxde: A deep learning library for solving differential equations.
\newblock \emph{SIAM review}, 63\penalty0 (1):\penalty0 208--228, 2021.

\bibitem[Maddu et~al.(2022)Maddu, Sturm, M{\"u}ller, and Sbalzarini]{maddu2022inverse}
Suryanarayana Maddu, Dominik Sturm, Christian~L M{\"u}ller, and Ivo~F Sbalzarini.
\newblock Inverse dirichlet weighting enables reliable training of physics informed neural networks.
\newblock \emph{Machine Learning: Science and Technology}, 3\penalty0 (1):\penalty0 015026, 2022.

\bibitem[Meade~Jr and Fernandez(1994)]{meade1994numerical}
Andrew~J Meade~Jr and Alvaro~A Fernandez.
\newblock The numerical solution of linear ordinary differential equations by feedforward neural networks.
\newblock \emph{Mathematical and Computer Modelling}, 19\penalty0 (12):\penalty0 1--25, 1994.

\bibitem[Mehta et~al.(2024)Mehta, Diakonikolas, and Harchaoui]{mehta2024drago}
Ronak Mehta, Jelena Diakonikolas, and Zaid Harchaoui.
\newblock Drago: Primal-dual coupled variance reduction for faster distributionally robust optimization.
\newblock In \emph{The Thirty-eighth Annual Conference on Neural Information Processing Systems}, 2024.

\bibitem[Mohri et~al.(2019)Mohri, Sivek, and Suresh]{mohri2019agnostic}
Mehryar Mohri, Gary Sivek, and Ananda~Theertha Suresh.
\newblock Agnostic federated learning.
\newblock In \emph{International conference on machine learning}, pages 4615--4625. PMLR, 2019.

\bibitem[Nemirovski(2004)]{nemirovski2004prox}
Arkadi Nemirovski.
\newblock Prox-method with rate of convergence o (1/t) for variational inequalities with lipschitz continuous monotone operators and smooth convex-concave saddle point problems.
\newblock \emph{SIAM Journal on Optimization}, 15\penalty0 (1):\penalty0 229--251, 2004.

\bibitem[Nemirovskij and Yudin(1983)]{nemirovskij1983problem}
Arkadij~Semenovi{\v{c}} Nemirovskij and David~Borisovich Yudin.
\newblock Problem complexity and method efficiency in optimization.
\newblock 1983.

\bibitem[Nouiehed et~al.(2019)Nouiehed, Sanjabi, Huang, Lee, and Razaviyayn]{nouiehed2019solving}
Maher Nouiehed, Maziar Sanjabi, Tianjian Huang, Jason~D Lee, and Meisam Razaviyayn.
\newblock Solving a class of non-convex min-max games using iterative first order methods.
\newblock \emph{Advances in Neural Information Processing Systems}, 32, 2019.

\bibitem[Patankar and Spalding(1983)]{patankar1983calculation}
Suhas~V Patankar and D~Brian Spalding.
\newblock A calculation procedure for heat, mass and momentum transfer in three-dimensional parabolic flows.
\newblock In \emph{Numerical prediction of flow, heat transfer, turbulence and combustion}, pages 54--73. Elsevier, 1983.

\bibitem[Raissi et~al.(2019)Raissi, Perdikaris, and Karniadakis]{raissi2019physics}
Maziar Raissi, Paris Perdikaris, and George~E Karniadakis.
\newblock Physics-informed neural networks: A deep learning framework for solving forward and inverse problems involving nonlinear partial differential equations.
\newblock \emph{Journal of Computational physics}, 378:\penalty0 686--707, 2019.

\bibitem[Rockafellar(2015)]{rockafellar2015convex}
Ralph~Tyrell Rockafellar.
\newblock Convex analysis:(pms-28).
\newblock 2015.

\bibitem[Saad and Van Der~Vorst(2000)]{saad2000iterative}
Yousef Saad and Henk~A Van Der~Vorst.
\newblock Iterative solution of linear systems in the 20th century.
\newblock \emph{Journal of Computational and Applied Mathematics}, 123\penalty0 (1-2):\penalty0 1--33, 2000.

\bibitem[Sener and Koltun(2018)]{sener2018multi}
Ozan Sener and Vladlen Koltun.
\newblock Multi-task learning as multi-objective optimization.
\newblock \emph{Advances in neural information processing systems}, 31, 2018.

\bibitem[Son et~al.(2023)Son, Cho, and Hwang]{son2023enhanced}
Hwijae Son, Sung~Woong Cho, and Hyung~Ju Hwang.
\newblock Enhanced physics-informed neural networks with augmented lagrangian relaxation method (al-pinns).
\newblock \emph{Neurocomputing}, 548:\penalty0 126424, 2023.

\bibitem[Strang and Fix(1971)]{strang1971analysis}
Gilbert Strang and George~Joseph Fix.
\newblock An analysis of the finite element method.
\newblock \emph{(No Title)}, 1971.

\bibitem[Thekumparampil et~al.(2019)Thekumparampil, Jain, Netrapalli, and Oh]{thekumparampil2019efficient}
Kiran~K Thekumparampil, Prateek Jain, Praneeth Netrapalli, and Sewoong Oh.
\newblock Efficient algorithms for smooth minimax optimization.
\newblock \emph{Advances in neural information processing systems}, 32, 2019.

\bibitem[Wang et~al.(2021)Wang, Teng, and Perdikaris]{wang2021understanding}
Sifan Wang, Yujun Teng, and Paris Perdikaris.
\newblock Understanding and mitigating gradient flow pathologies in physics-informed neural networks.
\newblock \emph{SIAM Journal on Scientific Computing}, 43\penalty0 (5):\penalty0 A3055--A3081, 2021.

\bibitem[Wang et~al.(2022)Wang, Yu, and Perdikaris]{wang2022and}
Sifan Wang, Xinling Yu, and Paris Perdikaris.
\newblock When and why pinns fail to train: A neural tangent kernel perspective.
\newblock \emph{Journal of Computational Physics}, 449:\penalty0 110768, 2022.

\bibitem[Wang et~al.(2020)Wang, Tsvetkov, Firat, and Cao]{wang2020gradient}
Zirui Wang, Yulia Tsvetkov, Orhan Firat, and Yuan Cao.
\newblock Gradient vaccine: Investigating and improving multi-task optimization in massively multilingual models.
\newblock \emph{arXiv preprint arXiv:2010.05874}, 2020.

\bibitem[Xiang et~al.(2022)Xiang, Peng, Liu, and Yao]{xiang2022self}
Zixue Xiang, Wei Peng, Xu~Liu, and Wen Yao.
\newblock Self-adaptive loss balanced physics-informed neural networks.
\newblock \emph{Neurocomputing}, 496:\penalty0 11--34, 2022.

\bibitem[Xu et~al.(2021)Xu, Zhang, Zhang, and Mandic]{xu2021convergence}
Dongpo Xu, Shengdong Zhang, Huisheng Zhang, and Danilo~P Mandic.
\newblock Convergence of the rmsprop deep learning method with penalty for nonconvex optimization.
\newblock \emph{Neural Networks}, 139:\penalty0 17--23, 2021.

\bibitem[Xu et~al.(2023)Xu, Zhang, Xu, and Lan]{xu2023unified}
Zi~Xu, Huiling Zhang, Yang Xu, and Guanghui Lan.
\newblock A unified single-loop alternating gradient projection algorithm for nonconvex--concave and convex--nonconcave minimax problems.
\newblock \emph{Mathematical Programming}, 201\penalty0 (1):\penalty0 635--706, 2023.

\bibitem[Yakubov and Yakubov(1999)]{yakubov1999differential}
Yakov Yakubov and Sasun Yakubov.
\newblock \emph{Differential-operator equations: ordinary and partial differential equations}, volume 103.
\newblock CRC Press, 1999.

\bibitem[Yao et~al.(2023)Yao, Su, Hao, Liu, Su, and Zhu]{yao2023multiadam}
Jiachen Yao, Chang Su, Zhongkai Hao, Songming Liu, Hang Su, and Jun Zhu.
\newblock Multiadam: Parameter-wise scale-invariant optimizer for multiscale training of physics-informed neural networks.
\newblock In \emph{International Conference on Machine Learning}, pages 39702--39721. PMLR, 2023.

\bibitem[Yu et~al.(2018)]{yu2018deep}
Bing Yu et~al.
\newblock The deep ritz method: a deep learning-based numerical algorithm for solving variational problems.
\newblock \emph{Communications in Mathematics and Statistics}, 6\penalty0 (1):\penalty0 1--12, 2018.

\bibitem[Yu et~al.(2022)Yu, Lu, Meng, and Karniadakis]{yu2022gradient}
Jeremy Yu, Lu~Lu, Xuhui Meng, and George~Em Karniadakis.
\newblock Gradient-enhanced physics-informed neural networks for forward and inverse pde problems.
\newblock \emph{Computer Methods in Applied Mechanics and Engineering}, 393:\penalty0 114823, 2022.

\bibitem[Yu et~al.(2020)Yu, Kumar, Gupta, Levine, Hausman, and Finn]{yu2020gradient}
Tianhe Yu, Saurabh Kumar, Abhishek Gupta, Sergey Levine, Karol Hausman, and Chelsea Finn.
\newblock Gradient surgery for multi-task learning.
\newblock \emph{Advances in neural information processing systems}, 33:\penalty0 5824--5836, 2020.

\bibitem[Zhang and Yang(2021)]{zhang2021survey}
Yu~Zhang and Qiang Yang.
\newblock A survey on multi-task learning.
\newblock \emph{IEEE transactions on knowledge and data engineering}, 34\penalty0 (12):\penalty0 5586--5609, 2021.

\bibitem[Zhao et~al.(2023)Zhao, Ding, and Prakash]{zhao2023pinnsformer}
Zhiyuan Zhao, Xueying Ding, and B~Aditya Prakash.
\newblock Pinnsformer: A transformer-based framework for physics-informed neural networks.
\newblock \emph{arXiv preprint arXiv:2307.11833}, 2023.

\bibitem[Zhu and Zabaras(2018)]{zhu2018bayesian}
Yinhao Zhu and Nicholas Zabaras.
\newblock Bayesian deep convolutional encoder--decoder networks for surrogate modeling and uncertainty quantification.
\newblock \emph{Journal of Computational Physics}, 366:\penalty0 415--447, 2018.

\end{thebibliography}
\begin{appendixpart}
\section{Additional Experiments}\label{ap:D}
In this section, we provide additional information to accompany the work. In Section \ref{sec:exps}, we mentioned that \texttt{Adam}+\texttt{Adam} and \texttt{RMSProp}+\texttt{RMSProp} were less effective compared to \texttt{Adam}+\texttt{RMSProp}. Below we present the results of training a vanilla \textit{PINN} on \textit{Poisson 2d-C}.
\begin{table}[htbp]
\centering
\renewcommand{\arraystretch}{1.3}
\setlength{\tabcolsep}{4pt} 
\begin{tabularx}{\linewidth}{@{} l *{3}{>{\centering\arraybackslash}X} @{}}
\toprule
\textbf{Approach (Descent-Ascent)} & \texttt{Adam+RMSProp} & \texttt{Adam+Adam} & \texttt{RMSProp+RMSProp} \\
\midrule
\textbf{L2RE} & \protect\colorbox{myblue!20}{\textbf{8.16E-3}} &  \protect\colorbox{mylightblue!20}{4.45E-2} & 6.02E-1 \\
\bottomrule
\end{tabularx}
\caption{Comparison of different approaches to incorporating adaptivity in Algorithm \ref{alg:bgda}. \textbf{L2RE} is used as a quality metric. We highlight the \protect\colorbox{myblue!20}{\textbf{best}} and the \protect\colorbox{mylightblue!20}{second best} results}
\label{tab:comparison_adaptivity}
\end{table}
We hypothesize that \texttt{Adam}+\texttt{RMSProp} shows the best quality because \texttt{Adam} allows to account for the poor loss landscape in $\theta$ via gradient smoothing, while the landscape in $\pi$ is strongly convex, and steps along the current gradient are more appropriate.

In addition, we use more modern \textit{PINN} architectures provided in \textit{PINNacle} \citep{hao2023pinnacle} to validate the theoretical insights. Below we summarize their key features:
\begin{itemize}
    \item \textbf{\textit{gPINN}.} It is known that the residual $(\mathcal{R}_i[u]-f_i)(x)$ must be zero inside the domain. Consequently, its derivative must also be equal to zero. This approach proposes to modify the objective by adding $\|\nicefrac{\partial}{\partial x}(\mathcal{R}_i[u]-f_i)(x)\|^2$ as a regularization. In \citep{yu2022gradient}, it is shown that \textit{gPINN} 
    has improved quality of the approximation inside the domain $\Omega$. 
    \item \textbf{\textit{GAAF}.} This architecture relies on adaptive activation functions (both layer- and neuron-wise). \citep{jagtap2020adaptive} demonstrates the advantages of this approach over vanilla \textit{PINN}s.
    \item \textbf{\textit{LAAF}.} Considers \textit{GAAF} with slope recovery term. For the details, see \citep{jagtap2020locally}.
\end{itemize}
Below we provide the comparison of the best known \textit{L2RE}s with ones provided by our approach.
\definecolor{best}{RGB}{0, 0, 139}
\begin{table}[htbp]
\centering
\scriptsize
\renewcommand{\arraystretch}{1.2}
\setlength{\tabcolsep}{4pt}
\caption{Training modern \textit{PINN} architectures via \texttt{AdaptiveBGDA}. In all experiments, the model is trained to the performance limit. \textbf{L2RE} is used as a quality metric. We highlight the \protect\colorbox{best!20}{\textbf{best}} results for each PDE and architecture.}
\begin{tabular}{llcccccc}
\toprule
\textbf{PDE} & & \multicolumn{2}{c}{\textit{gPINN}} & \multicolumn{2}{c}{\textit{LAAF}} & \multicolumn{2}{c}{\textit{GAAF}} \\
\cmidrule(lr){3-4} \cmidrule(lr){5-6} \cmidrule(lr){7-8}
 & & Best & Ours & Best & Ours & Best & Ours \\
\midrule
\multirow{2}{*}{Burgers}
 & 1d-C & 2.16E-1 & \cellcolor{best!20}\textbf{1.36e-2} & 1.43E-2 & \cellcolor{best!20}\textbf{1.30E-2} & 5.20E-2 & \cellcolor{best!20}\textbf{1.30E-2} \\
 & 2d-C & \cellcolor{best!20}\textbf{3.27E-1} & 5.11E-1 & \cellcolor{best!20}\textbf{2.77E-1} & 4.42E-1 & \cellcolor{best!20}\textbf{2.95E-1} & 5.09E-1 \\
 
\midrule
\multirow{4}{*}{Poisson}
 & 2d-C & 6.87E-1 & \cellcolor{best!20}\textbf{5.85E-1} & 7.68E-1 & \cellcolor{best!20}\textbf{1.38E-2} & 6.04E-1 & \cellcolor{best!20}\textbf{4.37E-3} \\
 & 2d-CG & 7.92E-1 & \cellcolor{best!20}\textbf{4.45E-1} & 4.80E-1 & \cellcolor{best!20}\textbf{1.11E-2} & 8.71E-1 & \cellcolor{best!20}\textbf{2.82E-2} \\
 & 3d-CG & \cellcolor{best!20}\textbf{4.85E-1} & 5.65E-1 & 5.79E-1 & \cellcolor{best!20}\textbf{5.43E-2} & 5.02E-1 & \cellcolor{best!20}\textbf{9.22E-2} \\
 & 2d-MS & 6.16E-1 & \cellcolor{best!20}\textbf{4.55E-1} & 5.93E-1 & \cellcolor{best!20}\textbf{3.72E-1} & 9.31E-1 & \cellcolor{best!20}\textbf{4.07E-1} \\
 
\midrule
\multirow{4}{*}{Heat}
 & 2d-VC & 2.12E+0 & \cellcolor{best!20}\textbf{1.01E+0} & 6.42E-1 & \cellcolor{best!20}\textbf{2.57E-1} & 8.49E-1 & \cellcolor{best!20}\textbf{7.03E-1} \\
 & 2d-MS & 1.13E-1 & \cellcolor{best!20}\textbf{3.95E-2} & 7.40E-2 & \cellcolor{best!20}\textbf{1.85E-2} & 9.85E-1 & \cellcolor{best!20}\textbf{6.67E-2} \\
 & 2d-CG & \cellcolor{best!20}\textbf{9.38E-2} & 1.09E-1 & \cellcolor{best!20}\textbf{2.39E-2} & 4.06E-2 & 4.61E-1 & \cellcolor{best!20}\textbf{1.18E-2} \\
 & 2d-LT & 1.00E+0 & \cellcolor{best!20}\textbf{9.99E-1} & 9.99E-1 & \cellcolor{best!20}\textbf{9.98E-1} & 9.99E-1 & \cellcolor{best!20}\textbf{9.98E-1} \\
 
\midrule
\multirow{3}{*}{NS}
 & 2d-C & 7.70E-2 & \cellcolor{best!20}\textbf{6.22E-2} & \cellcolor{best!20}\textbf{3.60E-2} & 8.14E-2 & 3.79E-2 & \cellcolor{best!20}\textbf{2.55E-2} \\
 & 2d-CG & 1.54E-1 & \cellcolor{best!20}\textbf{1.11E-1} & \cellcolor{best!20}\textbf{8.42E-2} & 1.25E-1 & 1.74E-1 & \cellcolor{best!20}\textbf{1.06E-1} \\
 & 2d-LT & 9.95E-1 & \cellcolor{best!20}\textbf{9.63E-1} & \cellcolor{best!20}\textbf{9.98E-1} & 9.99E-1 & 9.99E-1 & 9.99E-1 \\
 
\midrule
\multirow{3}{*}{Wave}
 & 1d-C & 5.56E-1 & \cellcolor{best!20}\textbf{6.95E-2} & 4.54E-1 & \cellcolor{best!20}\textbf{2.52E-2} & 6.77E-1 & \cellcolor{best!20}\textbf{2.97E-2} \\
 & 2d-CG & 8.14E-1 & \cellcolor{best!20}\textbf{7.82E-1} & 8.10E-1 & \cellcolor{best!20}\textbf{7.86E-1} & 7.94E-1 & \cellcolor{best!20}\textbf{7.81E-1} \\
 & 2d-MS & 1.02E+0 & \cellcolor{best!20}\textbf{9.09E-1} & 1.06E+0 & \cellcolor{best!20}\textbf{9.99E-1} & 1.06E+0 & \cellcolor{best!20}\textbf{9.99E-1} \\
 
\midrule
\multirow{2}{*}{Chaotic}
 & GS & 2.48E-1 & \cellcolor{best!20}\textbf{9.30E-2} & \cellcolor{best!20}\textbf{9.47E-2} & 9.49E-2 & 9.46E-2 & \cellcolor{best!20}\textbf{9.32E-2} \\
 & KS & 9.94E-1 & \cellcolor{best!20}\textbf{9.68E-1} & 1.01E+0 & \cellcolor{best!20}\textbf{9.99E-1} & 1.00E+0 & \cellcolor{best!20}\textbf{9.99E-1} \\
 
\midrule
\multirow{2}{*}{High dim}
 & PNd & 5.05E-3 & \cellcolor{best!20}\textbf{1.65E-3} & 4.14E-3 & \cellcolor{best!20}\textbf{8.00E-4} & 7.75E-2 & \cellcolor{best!20}\textbf{1.57E-3} \\
 & HNd & 3.17E-1 & \cellcolor{best!20}\textbf{9.00E-4} & 5.22E-1 & \cellcolor{best!20}\textbf{3.20E-4} & 5.21E-1 & \cellcolor{best!20}\textbf{3.20E-4} \\
 
\midrule
\multirow{2}{*}{Inverse}
 & PInv & \cellcolor{best!20}\textbf{8.03E-2} & 8.45E-1 & 1.30E-1 & \cellcolor{best!20}\textbf{9.49E-2} & 2.54E-1 & \cellcolor{best!20}\textbf{1.31E-1} \\
 & HInv & 4.84E+0 & \cellcolor{best!20}\textbf{6.71E-1} & 5.59E-1 & \cellcolor{best!20}\textbf{5.16E-2} & 2.12E-1 & \cellcolor{best!20}\textbf{5.97E-2} \\
 
\bottomrule
\end{tabular}
\label{tab:comparison_arch}
\end{table}
Table \ref{tab:comparison_arch} demonstrates that our scheme dominates not only for vanilla \textit{PINN}s, but also for novel architectures. The percentage of superiority is $81.8\%$ for \textit{gPINN}, $72.7\%$ for \textit{LAAF} and $90.1\%$ for \textit{GAAF}. Moreover, there is a significant drawdown only for \textit{Burgers 2d-C}. 

\section{Strong Concavity of the Objective}
In this section, we prove Lemma \ref{eq:lemma_1}. It follows obviously from the form of the objective (see \ref{eq:pinn_saddle}) and Assumption \ref{ass:p}.
\begin{lemma}(\textbf{Lemma \ref{eq:lemma_1}).}
    Consider the problem \eqref{eq:pinn_saddle} under Assumption \ref{ass:p}. Then, for every $\theta\in\Rd$ the function $\mathcal{L}(\theta,\pi)$ is \textbf{$\lambda$-strongly concave}, i.e. for all $\pi_1,\pi_2\in S$ it satisfies
    \begin{align*}
        \mathcal{L}(\theta,\pi_1)\leq\mathcal{L}(\theta,\pi_2)+\langle\nabla_{\psi}\mathcal{L}(\theta,\pi_2),\pi_1-\pi_2\rangle-\frac{\lambda}{2}\left(D_{\psi}(\pi_1,\pi_2)+D_{\psi}(\pi_2,\pi_1)\right).
    \end{align*}
\end{lemma}
\begin{proof}
    Note that $\nabla^2_{\pi}\mathcal{L}(\theta,\pi)=-\lambda\nabla^2\psi(\pi)$. The function $\mathcal{L}(\theta,\pi)$ is $\mu$-strongly concave related to $D_{\psi}$, if $\nabla^2_{\pi}\mathcal{L}(\theta,\pi)\preceq-\mu\nabla^2\psi(\pi)$ \citep{lu2018relatively}. Therefore, the objective is $\lambda$-stongly relatively concave. 
\end{proof}

\section{Proof of Lemma \ref{lemma:distance}}\label{ap:A}
We begin the presentation of the analysis with a key result guaranteeing convergence. It demonstrates that the distance between $\pi^t$ and the exact maximum of $\pi^*(\theta^t)$ has a suitable dynamics with increasing $t$.
\begin{lemma}\textbf{(Lemma \ref{lemma:distance}).}
    Consider the problem \eqref{eq:pinn_saddle} under Assumptions \ref{ass:theta}, \ref{ass:p}. Then, Algorithm \ref{alg:bgda} with tuning
    \begin{align*}
        \gamma_{\pi}=\frac{\lambda}{4L^2},\quad\gamma_{\theta}\leq\frac{1}{184\kappa^4L}
    \end{align*}
    produces such $\{(\theta^t,\pi^t)\}_{t=1}^T$, that
    \begin{align*}
        D_{\psi}(\pi^*(\theta^{t+1}),\pi^{t+1})\leq\left(1-\frac{1}{64\kappa^2}\right)D_{\psi}(\pi^*(\theta^t),\pi^{t})+264\gamma_{\theta}^2\kappa^6\|\nabla\Phi(\theta^t)\|^2,
    \end{align*}
    where $\kappa=\nicefrac{L}{\lambda}$ is the condition number of $\mathcal{L}(\theta,\pi)$ in $\pi$.
\end{lemma}
\begin{proof}
    Before proceeding to the proof, let us recall the three-point identity. It plays a key role in the analysis of Bregman methods.
    \begin{align}\label{eq:three-point}
        D_{\psi}(x,y)-D_{\psi}(x,z)-D_{\psi}(z,y)=\langle 
\nabla\psi(z)-\nabla\psi(y),x-z \rangle.
    \end{align}
    To begin, we use \eqref{eq:three-point} in the form
    \begin{align}\label{eq:start_of_lemma}
    \begin{split}
        D_{\psi}(\pi^*(\theta^{t+1}),\pi^{t+1})=&D_{\psi}(\pi^*(\theta^{t+1}),\pi^*(\theta^{t}))+D_{\psi}(\pi^*(\theta^{t}),\pi^{t+1})\\&+\langle 
\nabla\psi(\pi^*(\theta^{t}))-\nabla\psi(\pi^{t+1}),\pi^*(\theta^{t+1})-\pi^*(\theta^{t}) \rangle.
    \end{split}
    \end{align}
    Further, we write the optimality condition for Line \ref{line:ascent}:
    \begin{align*}
        \left\langle -\gamma_{\pi}\nabla_{\pi}\mathcal{L}(\theta^t,\pi^t)+[\nabla\psi(\pi^{t+1})-\nabla\psi(\pi^t)],\pi^*(\theta^t)-\pi^{t+1} \right\rangle\geq0.
    \end{align*}
    Applying \eqref{eq:three-point}, we obtain
    \begin{align*}
        -\gamma_{\pi}\left\langle\nabla_{\pi}\mathcal{L}(\theta^t,\pi^t),\pi^*(\theta^t)-\pi^{t+1}\right\rangle+D_{\psi}(\pi^*(\theta^t),\pi^t)-D_{\psi}(\pi^*(\theta^t),\pi^{t+1})-D_{\psi}(\pi^{t+1},\pi^t)\geq0.
    \end{align*}
    After re-arranging the terms, we get
    \begin{align}\label{eq:2}
        D_{\psi}(\pi^*(\theta^t),\pi^{t+1})\leq D_{\psi}(\pi^*(\theta^t),\pi^t)-D_{\psi}(\pi^{t+1},\pi^t)-\gamma_{\pi}\left\langle\nabla_{\pi}\mathcal{L}(\theta^t,\pi^t),\pi^*(\theta^t)-\pi^{t+1}\right\rangle.
    \end{align}
    Since $\pi^*(\theta^t)$ is the exact maximum of $\mathcal{L}(\theta^t,\pi)$ in $\pi$, there is another optimility condition
    \begin{align*}
        \gamma_{\pi}\left\langle \nabla_{\pi}\mathcal{L}(\theta^t,\pi^*(\theta^t)),\pi^*(\theta^t)-\pi \right\rangle\geq0.
    \end{align*}
    Substituting $\pi=\pi^{t+1}$ and summing it with \eqref{eq:2}, we derive
    \begin{align*}
        D_{\psi}(\pi^*(\theta^t),\pi^{t+1})\leq& D_{\psi}(\pi^*(\theta^t),\pi^t)-D_{\psi}(\pi^{t+1},\pi^t)\\&+\gamma_{\pi}\left\langle\nabla_{\pi}\mathcal{L}(\theta^t,\pi^*(\theta^t))-\nabla_{\pi}\mathcal{L}(\theta^t,\pi^t),\pi^*(\theta^t)-\pi^{t+1}\right\rangle\\\leq&D_{\psi}(\pi^*(\theta^t),\pi^t)-D_{\psi}(\pi^{t+1},\pi^t)\\&+\gamma_{\pi}\left\langle\nabla_{\pi}\mathcal{L}(\theta^t,\pi^*(\theta^t))-\nabla_{\pi}\mathcal{L}(\theta^t,\pi^t),\pi^*(\theta^t)-\pi^{t}\right\rangle\\&+\gamma_{\pi}\left\langle\nabla_{\pi}\mathcal{L}(\theta^t,\pi^*(\theta^t))-\nabla_{\pi}\mathcal{L}(\theta^t,\pi^t),\pi^t-\pi^{t+1}\right\rangle.
    \end{align*}
    Now, we are going to utilize the strong concavity of $\mathcal{L}(\theta,\pi)$ in $\pi$:
    \begin{align*}
        \gamma_{\pi}\left\langle\nabla_{\pi}\mathcal{L}(\theta^t,\pi^*(\theta^t))-\nabla_{\pi}\mathcal{L}(\theta^t,\pi^t),\pi^*(\theta^t)-\pi^{t}\right\rangle\leq\frac{-\gamma_{\pi}\lambda}{2}D_{\psi}(\pi^*(\theta^t),\pi^t).
    \end{align*}
    Thus, we have
    \begin{align*}
        D_{\psi}(\pi^*(\theta^t),\pi^{t+1})\leq&\left(1-\frac{\gamma_{\pi}\lambda}{2}\right)D_{\psi}(\pi^*(\theta^t),\pi^t)-D_{\psi}(\pi^{t+1},\pi^t)\\&+\gamma_{\pi}\left\langle\nabla_{\pi}\mathcal{L}(\theta^t,\pi^*(\theta^t))-\nabla_{\pi}\mathcal{L}(\theta^t,\pi^t),\pi^t-\pi^{t+1}\right\rangle.
    \end{align*}
    Next, we apply Cauchy-Schwartz inequality to the scalar product and obtain
    \begin{align*}
        D_{\psi}(\pi^*(\theta^t),\pi^{t+1})\leq&\left(1-\frac{\gamma_{\pi}\lambda}{2}\right)D_{\psi}(\pi^*(\theta^t),\pi^t)-D_{\psi}(\pi^{t+1},\pi^t)\\&+\frac{\gamma_{\pi}\alpha}{2}\|\nabla_{\pi}\mathcal{L}(\theta^t,\pi^*(\theta^t))-\nabla_{\pi}\mathcal{L}(\theta^t,\pi^t)\|^2+\frac{\gamma_{\pi}}{2\alpha}\|\pi^t-\pi^{t+1}\|^2.
    \end{align*}
    Using $L$-smoothness of $\mathcal{L}$ (see Assumption \ref{ass:theta}), we obtain
    \begin{align*}
        D_{\psi}(\pi^*(\theta^t),\pi^{t+1})\leq&\left(1-\frac{\gamma_{\pi}\lambda}{2}\right)D_{\psi}(\pi^*(\theta^t),\pi^t)-D_{\psi}(\pi^{t+1},\pi^t)\\&+\frac{\gamma_{\pi}\alpha L^2}{2}\|\pi^*(\theta^t)-\pi^t\|^2+\frac{\gamma_{\pi}}{2\alpha}\|\pi^t-\pi^{t+1}\|^2.
    \end{align*}
    Since $\psi$ is $1$-strongly convex (see Assumption \ref{ass:p}), we have
    \begin{align*}
        \frac{1}{2}\|\pi_1-\pi_2\|^2\leq D_{\psi}(\pi_1,\pi_2).
    \end{align*}
    Thus,
    \begin{align*}
        D_{\psi}(\pi^*(\theta^t),\pi^{t+1})\leq&\left(1-\frac{\gamma_{\pi}\lambda}{2}\right)D_{\psi}(\pi^*(\theta^t),\pi^t)-D_{\psi}(\pi^{t+1},\pi^t)\\&+\gamma_{\pi}\alpha L^2D_{\psi}(\pi^*(\theta^t),\pi^t)+\frac{\gamma_{\pi}}{\alpha}D_{\psi}(\pi^t,\pi^{t+1}).
    \end{align*}
    Choose $\alpha=\gamma_{\pi}$. We can derive
    \begin{align*}
        D_{\psi}(\pi^*(\theta^{t}),\pi^{t+1})\leq\left(1-\frac{\gamma_{\pi}\lambda}{2}+\gamma_{\pi}^2L^2\right)D_{\psi}(\pi^*(\theta^t),\pi^t).
    \end{align*}
    Since $\gamma_{\pi}=\nicefrac{\lambda}{4L^2}$, we have
    \begin{align}\label{eq:4}
        D_{\psi}(\pi^*(\theta^{t}),\pi^{t+1})\leq\left(1-\frac{1}{16\kappa^2}\right)D_{\psi}(\pi^*(\theta^t),\pi^t).
    \end{align}
    Let us return to \eqref{eq:start_of_lemma}. Note that
    \begin{align*}
        \nabla\psi(\pi^*(\theta^t))-\nabla\psi(\pi^{t+1})=\frac{1}{\lambda}\left(\nabla_{\pi}\mathcal{L}(\theta^t,\pi^{t+1})-\nabla_{\pi}\mathcal{L}(\theta^t,\pi^*(\theta^t))\right).
    \end{align*}
    Thus, there is
    \begin{align*}
        D_{\psi}(\pi^*(\theta^{t+1}),\pi^{t+1})=&D_{\psi}(\pi^*(\theta^{t+1}),\pi^*(\theta^{t}))+D_{\psi}(\pi^*(\theta^{t}),\pi^{t+1})\\&+\frac{1}{\lambda}\langle 
\nabla_{\pi}\mathcal{L}(\theta^t,\pi^{t+1})-\nabla_{\pi}\mathcal{L}(\theta^t,\pi^*(\theta^t)),\pi^*(\theta^{t+1})-\pi^*(\theta^{t}) \rangle\\\leq&D_{\psi}(\pi^*(\theta^{t+1}),\pi^*(\theta^{t}))+D_{\psi}(\pi^*(\theta^{t}),\pi^{t+1})\\&+\frac{\alpha L^2}{\lambda}D_{\psi}(\pi^*(\theta^t),\pi^{t+1})+\frac{1}{\lambda\alpha}D_{\psi}(\pi^*({\theta^{t+1}}),\pi^*(\theta^t)).
    \end{align*}
    Let us choose $\alpha=\nicefrac{\lambda^3}{32L^4}$. With such a choice, we have
    \begin{align*}
        D_{\psi}(\pi^*(\theta^{t+1}),\pi^{t+1})\leq33\kappa^4D_{\psi}(\pi^*({\theta^{t+1}}),\pi^*(\theta^t))+\left(1+\frac{1}{32\kappa^2}\right)D_{\psi}(\pi^*(\theta^t),\pi^{t+1}).
    \end{align*}
    To deal with $D_{\psi}(\pi^*(\theta^t),\pi^{t+1})$, we utilize \eqref{eq:4}. As a result, we obtain
    \begin{align}\label{eq:7}
        D_{\psi}(\pi^*(\theta^{t+1}),\pi^{t+1})\leq33\kappa^4D_{\psi}(\pi^*({\theta^{t+1}}),\pi^*(\theta^t))+\left(1-\frac{1}{32\kappa^2}\right)D_{\psi}(\pi^*(\theta^t),\pi^{t}).
    \end{align}
    The rest thing is to prove that the descent step does not dramatically change the distance between the optimal values of weights.
    Let us write down two optimality conditions:
    \begin{align*}
        &\langle \nabla_{\pi}\mathcal{L}(\theta^t,\pi^*(\theta^t)),\pi-\pi^*(\theta^t) \rangle\leq0,\\
        &\langle \nabla_{\pi}\mathcal{L}(\theta^{t+1},\pi^*(\theta^{t+1})),\pi-\pi^*(\theta^{t+1}) \rangle\leq0.
    \end{align*}
    Let us substitute $\pi=\pi^*(\theta^{t+1})$ into the first inequality and $\pi=\pi^*(\theta^{t})$ into the second one. When summing them up, we have
    \begin{align}\label{eq:5}
        \langle \nabla_{\pi}\mathcal{L}(\theta^t,\pi^*(\theta^t))-\nabla_{\pi}\mathcal{L}(\theta^{t+1},\pi^*(\theta^{t+1})), \pi^*(\theta^{t+1})-\pi^*(\theta^t) \rangle\leq0.
    \end{align}
    On the other hand, we can take advantage of the strong concavity of the objective (see Lemma \ref{eq:lemma_1}) and write
    \begin{align}\label{eq:6}
        &\langle \nabla_{\pi}\mathcal{L}(\theta^t,\pi^*(\theta^{t+1}))-\nabla_{\pi}\mathcal{L}(\theta^t,\pi^*(\theta^t)),\pi^*(\theta^{t+1})-\pi^*(\theta^{t}) \rangle\\&\leq-\frac{\lambda}{2}\left[D_{\psi}(\pi^*(\theta^t),\pi^*(\theta^{t+1}))+D_{\psi}(\pi^*(\theta^{t+1}),\pi^*(\theta^t))\right].
    \end{align}
    Combining \eqref{eq:5} and \eqref{eq:6}, we obtain
    \begin{align*}
        \frac{\lambda^2}{4}\left[D_{\psi}(\pi^*(\theta^t),\pi^*(\theta^{t+1}))+D_{\psi}(\pi^*(\theta^{t+1}),\pi^*(\theta^t))\right]^2\leq L^2\|\pi^*(\theta^{t+1})-\pi^*(\theta^{t})\|^2\cdot\|\theta^{t+1}-\theta^t\|^2.
    \end{align*}
    Applying the strong convexity of distance generating function (Assumption \ref{ass:p}) and re-arranging terms, we obtain
    \begin{align*}
        D_{\psi}(\pi^*(\theta^t),\pi^*(\theta^{t+1}))+D_{\psi}(\pi^*(\theta^{t+1}),\pi^*(\theta^t))\leq&4\kappa^2\|\theta^{t+1}-\theta^t\|^2\leq4\gamma_{\theta}^2\kappa^2\|\nabla_{\theta}\mathcal{L}(\theta^t,\pi^t)\|^2.
    \end{align*}
    Next, we ass and subtract $\nabla\Phi(\theta^t)$ and apply Assumption \ref{ass:theta}. We obtain
    \begin{align*}
        D_{\psi}(\pi^*(\theta^t),\pi^*(\theta^{t+1}))+D_{\psi}(\pi^*(\theta^{t+1}),\pi^*(\theta^t))\leq&16\gamma_{\theta}^2\kappa^2L^2D_{\psi}(\pi^*(\theta^t),\pi^t)+8\gamma_{\theta}^2\kappa^2\|\nabla\Phi(\theta^t)\|^2.
    \end{align*}
    Thus, \eqref{eq:7} transforms into 
    \begin{align*}
        D_{\psi}(\pi^*(\theta^{t+1}),\pi^{t+1})\leq\left(1-\frac{1}{32\kappa^2}+528\gamma_{\theta}^2\kappa^6L^2\right)D_{\psi}(\pi^*(\theta^t),\pi^{t})+264\gamma_{\theta}^2\kappa^6\|\nabla\Phi(\theta^t)\|^2.
    \end{align*}
    With $\gamma_{\theta}\leq\nicefrac{1}{184\kappa^4L}$, we obtain
    \begin{align*}
        D_{\psi}(\pi^*(\theta^{t+1}),\pi^{t+1})\leq\left(1-\frac{1}{64\kappa^2}\right)D_{\psi}(\pi^*(\theta^t),\pi^{t})+264\gamma_{\theta}^2\kappa^6\|\nabla\Phi(\theta^t)\|^2.
    \end{align*}
    This completes the proof.
\end{proof}

\section{Proof of Theorem \ref{th:general}}\label{ap:B}
\begin{theorem}\textbf{(Theorem \ref{th:general})}
    Consider the problem \eqref{eq:pinn_saddle} under Assumptions \ref{ass:theta}, \ref{ass:p}. Then, Algorithm \ref{alg:bgda} with tuning 
    \begin{align*}
        \gamma_{\pi}=\frac{\lambda}{4L^2},\quad\gamma_{\theta}\leq\sqrt{\frac{43}{92*33792}}\frac{1}{\kappa^4L}
    \end{align*}
    requires
    \begin{align*}
        \mathcal{O}\left(\frac{\kappa^4L\Delta+\kappa^2L^2D_{\psi}(\pi^*(\theta^0),\pi^0)}{\varepsilon^2}\right)\text{ iterations}
    \end{align*}
    to achieve an arbitrary $\varepsilon$-solution, where $\varepsilon^2 = \frac{1}{T}\sum_{t=1}^{T-1}\|\nabla\Phi(\theta^t)\|^2$, $\Delta=\Phi(\theta^0)-\Phi(\theta^*)$. $\kappa=\nicefrac{L}{\lambda}$.
\end{theorem}
\begin{proof}
    One can note that $\Phi$ is $3\kappa L$-smooth. Indeed,
    \begin{align*}
        \|\nabla\Phi(\theta_1)-\nabla\Phi(\theta_2)\|^2=&\|\nabla_{\theta}\mathcal{L}(\theta_1,\pi^*(\theta_1))-\nabla_{\theta}\mathcal{L}(\theta_2,\pi^*(\theta_2))\|^2\\\leq&  L^2\left[\|\theta_1-\theta_2\|^2+2D_{\psi}(\pi^*(\theta_1),\pi^*(\theta_2))\right]\leq L^2\left(1+4\kappa^2\right)\|\theta_1-\theta_2\|^2\\\leq&9\kappa^2L^2\|\theta_1-\theta_2\|^2.
    \end{align*}
    Thus, we can write
    \begin{align*}
        \Phi(\theta^{t+1})\leq&\Phi(\theta^t)+\langle \nabla\Phi(\theta^t), \theta^{t+1}-\theta^t \rangle+3\kappa L\|\theta^{t+1}-\theta^t\|^2\\\leq&\Phi(\theta^t)-\gamma_{\theta}\|\nabla\Phi(\theta^t)\|^2+3\gamma_{\theta}^2\kappa L\|\nabla_{\theta}\mathcal{L}(\theta^t,\pi^t)\|^2\\&+\gamma_{\theta}\langle \nabla\Phi(\theta^t)-\nabla_{\theta}\mathcal{L}(\theta^t,\pi^t),\nabla\Phi(\theta^t) \rangle\\\leq&\Phi(\theta^t)-\frac{\gamma_{\theta}}{2}\|\nabla\Phi(\theta^t)\|^2+3\gamma_{\theta}^2\kappa L\|\nabla_{\theta}\mathcal{L}(\theta^t,\pi^t)\|^2+\frac{\gamma_{\theta}}{2}\|\nabla\Phi(\theta^t)-\nabla_{\theta}\mathcal{L}(\theta^t,\pi^t)\|^2\\\leq&\Phi(\theta^t)-\left(\frac{\gamma_{\theta}}{2}- 6\gamma_{\theta}^2\kappa L\right)\|\nabla\Phi(\theta^t)\|^2+\left(\frac{\gamma_{\theta}}{2}+6\gamma_{\theta}^2\kappa L\right)\|\nabla\Phi(\theta^t)-\nabla_{\theta}\mathcal{L}(\theta^t,\pi^t)\|^2.
    \end{align*}
    Note that
    \begin{align*}
        -\left(\frac{\gamma_{\theta}}{2}- 6\gamma_{\theta}^2\kappa L\right)\leq-\frac{43\gamma_{\theta}}{92}.
    \end{align*}
    On the other hand,
    \begin{align*}
        \left(\frac{\gamma_{\theta}}{2}+6\gamma_{\theta}^2\kappa L\right)\leq\gamma_{\theta}.
    \end{align*}
    Thus, we have
    \begin{align*}
        \Phi(\theta^{t+1})\leq&\Phi(\theta^t)-\frac{43\gamma_{\theta}}{92}\|\nabla\Phi(\theta^t)\|^2+\gamma_{\theta}\|\nabla\Phi(\theta^t)-\nabla_{\theta}\mathcal{L}(\theta^t,\pi^t)\|^2\\\leq&\Phi(\theta^t)-\frac{43\gamma_{\theta}}{92}\|\nabla\Phi(\theta^t)\|^2+2\gamma_{\theta}L^2D_{\psi}(\pi^*({\theta^t}),\pi^t).
    \end{align*}
    Let us denote $\delta=1-\nicefrac{1}{64\kappa^2}$. Lemma \ref{lemma:distance} transforms into
    \begin{align*}
        D_{\psi}(\pi^*(\theta^t),\pi^t)\leq\delta^tD_{\psi}(\pi^*(\theta^0),\pi^0)+264\gamma_{\theta}^2\kappa^6\sum_{j=0}^{t-1}\delta^{t-1-j}\|\nabla\Phi(\theta^j)\|^2.
    \end{align*}
    Hence,
    \begin{align*}
        \Phi(\theta^{t+1})\leq&\Phi(\theta^t)-\frac{43\gamma_{\theta}}{92}\|\nabla\Phi(\theta^t)\|^2+2\gamma_{\theta}L^2\delta^tD_{\psi}(\pi^*(\theta^0),\pi^0)\\&+528\gamma_{\theta}^3\kappa^6L^2\sum_{j=0}^{t-1}\delta^{t-1-j}\|\nabla\Phi(\theta^j)\|^2.
    \end{align*}
    Let us sum up over the iterates $t$ and obtain
    \begin{align*}
        \Phi(\theta^{T})\leq&\Phi(\theta^0)-\frac{43\gamma_{\theta}}{92}\sum_{t=1}^{T-1}\|\nabla\Phi(\theta^t)\|^2+2\gamma_{\theta}L^2\sum_{t=1}^{T-1}\delta^tD_{\psi}(\pi^*(\theta^0),\pi^0)\\&+528\gamma_{\theta}^3\kappa^6L^2\sum_{t=1}^{T-1}\sum_{j=0}^{t-1}\delta^{t-1-j}\|\nabla\Phi(\theta^j)\|^2.
    \end{align*}
    Next, we use the property of geometric progression and write
    \begin{align*}
        \Phi(\theta^{T})\leq&\Phi(\theta^0)-\frac{43\gamma_{\theta}}{92}\sum_{t=1}^{T-1}\|\nabla\Phi(\theta^t)\|^2+128\gamma_{\theta}\kappa^2L^2D_{\psi}(\pi^*(\theta^0),\pi^0)\\&+33792\gamma_{\theta}^3\kappa^8L^2\sum_{t=1}^{T-1}\|\nabla\Phi(\theta^t)\|^2.
    \end{align*}
    Choosing $\gamma_{\theta}\leq\sqrt{\frac{43}{92*33792}}\frac{1}{\kappa^4L}$.
    Thus, we derive
    \begin{align*}
        \frac{1}{T}\sum_{t=1}^{T-1}\|\nabla\Phi(\theta^t)\|^2\leq\mathcal{O}\left(\frac{\kappa^4L\Delta_{\Phi}}{T}+\frac{\kappa^2L^2D_{\psi}(\pi^*(\theta^0),\pi^0)}{T}\right).
    \end{align*}
\end{proof}
\section{Enhanced Rates on Regularized Simplex}\label{ap:C}
The theory presented in Appendices \ref{ap:A}, \ref{ap:B} is constructed for and arbitrary Bregman divergence. This is the main reason for the deterioration of the theoretical guarantees compared to the Euclidean setting. In this section, we look towards the selection of the efficient approach for determining the set of weights $S$. We consider a classic approach of using a unit simplex $\triangle_1^{M-1}$:
\begin{align*}
    \triangle_1^{M-1}=\left\{ (\pi_1,\ldots,\pi_M):\pi_m\geq0,\sum_{m=1}^M\pi_m=1, \right\}.
\end{align*}
Note that $\psi(\pi)=-\sum_{m=1}^M\pi_m\log{\pi_m}$ goes to infinity at vertices of $\triangle_1^{M-1}$. Thus, one cannot guarantee smoothness of $\mathcal{L}(\theta,\pi)$ in $\pi$ for every fixed $\theta$. To avoid this, we propose to intersect the simplex by a euclidean ball. This approach is common in the literature \citep{mehta2024drago}. Thus, we deal with
\begin{align*}
    S=\triangle_1^{M-1}\cap B_{\|\cdot\|}(\mathcal{U},R),
\end{align*}
where $\mathcal{U}=(\nicefrac{1}{M},\ldots,\nicefrac{1}{M})^\top$.
\begin{lemma}\label{lemma:p_lip}
    The function $\mathcal{L}(\theta,\pi)$ is $L_{\pi}$-smooth in $\pi$, i.e. for all $\pi_1, \pi_2 \in S$ it satisfies 
    \begin{align*}
        \|\nabla \mathcal{L}(\theta,\pi_1) - \nabla \mathcal{L}(\theta,\pi_2)\| \leq L_{\pi}\|\pi_1-\pi_2\|^2.
    \end{align*}
    Moreover, under strong regularization ($R\ll1$), it is
    \begin{align*}
        L_{\pi}=\Theta(\lambda M^2R).
    \end{align*}
\end{lemma}
\begin{proof}
    Without loss of generality, consider $\pi=(a,b,\ldots,b)$, where $a=\min_m\pi_m$. Note that
    \begin{align*}
        \|\nabla^2_{\pi}\mathcal{L}(\theta,\pi)\|=\lambda\left\|\text{diag}\left(\frac{1}{\pi_1},\ldots,\frac{1}{\pi_M}\right)\right\|.
    \end{align*}
    Thus, we need to find $\max_{a\in\triangle_1^{M-1}}\frac{1}{a}$ with $\|\pi-\mathcal{U}\|^2\leq R^2$. Let us write
    \begin{align}\label{eq:8}
        \|\pi-\mathcal{U}\|^2=\left(a-\frac{1}{M}\right)^2+(M-1)\left(b-\frac{1}{M}\right)^2\leq R^2.
    \end{align}
    Consider $b=\frac{1-a}{M-1}$. Then \eqref{eq:8} transforms into
    \begin{align*}
        \left( a-\frac{1}{M} \right)^2+\frac{(1-aM)^2}{M^2(M-1)}\leq R^2.
    \end{align*}
    Solving the one-dimensional optimization problem, we find the Lipschitz constant of $\nabla_{\pi}\mathcal{L}(\theta,\pi)$. If $R\ll1$, then
    \begin{align*}
        L_{\pi}=\frac{\lambda}{\nicefrac{1}{M}-\Theta(R)}=\frac{\lambda M}{1-M\Theta(R)}\approx\Theta(\lambda M^2R).
    \end{align*}
\end{proof}
Note that this value is negligible. Indeed, $R\in(0,1)$, and $M$ in problems of mathematical physics (see \eqref{eq:problem}) is usually equal to $3$--$4$. Thus, if $\kappa_{\pi}=\nicefrac{L_{\pi}}{\lambda}$ appears in the estimate, it is comparable in magnitude to other constants hidden in the big-O.

Now let us move to an analysis with enhanced rate.
\begin{lemma}\label{lemma:distance_enh}
    Consider the problem \eqref{eq:pinn_saddle} under Assumptions \ref{ass:theta}, \ref{ass:p}. Let $S=\triangle_1^{M-1}\cap B_{\|\cdot\|}(\mathcal{U},R)$. Then, Algorithm \ref{alg:bgda} with tuning
    \begin{align*}
        \gamma_{\pi}=\frac{\lambda}{4L_{\pi}^2},\quad\gamma_{\theta}\leq\frac{1}{184\kappa_{\pi}^3\kappa L}
    \end{align*}
    produces such $\{(\theta^t,\pi^t)\}_{t=1}^T$, that
    \begin{align*}
        D_{\psi}(\pi^*(\theta^{t+1}),\pi^{t+1})\leq\left(1-\frac{1}{64\kappa_{\pi}^2}\right)D_{\psi}(\pi^*(\theta^t),\pi^{t})+264\gamma_{\theta}^2\kappa_{\pi}^4\kappa^2\|\nabla\Phi(\theta^t)\|^2,
    \end{align*}
    where $\kappa=\nicefrac{L}{\lambda}$, $\kappa_{\pi}=\nicefrac{L_{\pi}}{\lambda}$.
\end{lemma}
\begin{proof}
    To begin, we use \eqref{eq:three-point} in the form
    \begin{align}\label{eq:start_of_lemma_enh}
    \begin{split}
        D_{\psi}(\pi^*(\theta^{t+1}),\pi^{t+1})=&D_{\psi}(\pi^*(\theta^{t+1}),\pi^*(\theta^{t}))+D_{\psi}(\pi^*(\theta^{t}),\pi^{t+1})\\&+\langle 
\nabla\psi(\pi^*(\theta^{t}))-\nabla\psi(\pi^{t+1}),\pi^*(\theta^{t+1})-\pi^*(\theta^{t}) \rangle.
    \end{split}
    \end{align}
    Further, we write the optimality condition for Line \ref{line:ascent}:
    \begin{align*}
        \left\langle -\gamma_{\pi}\nabla_{\pi}\mathcal{L}(\theta^t,\pi^t)+[\nabla\psi(\pi^{t+1})-\nabla\psi(\pi^t)],\pi^*(\theta^t)-\pi^{t+1} \right\rangle\geq0.
    \end{align*}
    Applying \eqref{eq:three-point}, we obtain
    \begin{align*}
        -\gamma_{\pi}\left\langle\nabla_{\pi}\mathcal{L}(\theta^t,\pi^t),\pi^*(\theta^t)-\pi^{t+1}\right\rangle+D_{\psi}(\pi^*(\theta^t),\pi^t)-D_{\psi}(\pi^*(\theta^t),\pi^{t+1})-D_{\psi}(\pi^{t+1},\pi^t)\geq0.
    \end{align*}
    After re-arranging the terms, we get
    \begin{align}\label{eq:2_enh}
        D_{\psi}(\pi^*(\theta^t),\pi^{t+1})\leq D_{\psi}(\pi^*(\theta^t),\pi^t)-D_{\psi}(\pi^{t+1},\pi^t)-\gamma_{\pi}\left\langle\nabla_{\pi}\mathcal{L}(\theta^t,\pi^t),\pi^*(\theta^t)-\pi^{t+1}\right\rangle.
    \end{align}
    Since $\pi^*(\theta^t)$ is the exact maximum of $\mathcal{L}(\theta^t,\pi)$ in $\pi$, there is another optimility condition
    \begin{align*}
        \gamma_{\pi}\left\langle \nabla_{\pi}\mathcal{L}(\theta^t,\pi^*(\theta^t)),\pi^*(\theta^t)-\pi \right\rangle\geq0.
    \end{align*}
    Substituting $\pi=\pi^{t+1}$ and summing it with \eqref{eq:2_enh}, we derive
    \begin{align*}
        D_{\psi}(\pi^*(\theta^t),\pi^{t+1})\leq& D_{\psi}(\pi^*(\theta^t),\pi^t)-D_{\psi}(\pi^{t+1},\pi^t)\\&+\gamma_{\pi}\left\langle\nabla_{\pi}\mathcal{L}(\theta^t,\pi^*(\theta^t))-\nabla_{\pi}\mathcal{L}(\theta^t,\pi^t),\pi^*(\theta^t)-\pi^{t+1}\right\rangle\\\leq&D_{\psi}(\pi^*(\theta^t),\pi^t)-D_{\psi}(\pi^{t+1},\pi^t)\\&+\gamma_{\pi}\left\langle\nabla_{\pi}\mathcal{L}(\theta^t,\pi^*(\theta^t))-\nabla_{\pi}\mathcal{L}(\theta^t,\pi^t),\pi^*(\theta^t)-\pi^{t}\right\rangle\\&+\gamma_{\pi}\left\langle\nabla_{\pi}\mathcal{L}(\theta^t,\pi^*(\theta^t))-\nabla_{\pi}\mathcal{L}(\theta^t,\pi^t),\pi^t-\pi^{t+1}\right\rangle.
    \end{align*}
    Now, we are going to utilize the strong concavity of $\mathcal{L}(\theta,\pi)$ in $\pi$:
    \begin{align*}
        \gamma_{\pi}\left\langle\nabla_{\pi}\mathcal{L}(\theta^t,\pi^*(\theta^t))-\nabla_{\pi}\mathcal{L}(\theta^t,\pi^t),\pi^*(\theta^t)-\pi^{t}\right\rangle\leq\frac{-\gamma_{\pi}\lambda}{2}D_{\psi}(\pi^*(\theta^t),\pi^t).
    \end{align*}
    Thus, we have
    \begin{align*}
        D_{\psi}(\pi^*(\theta^t),\pi^{t+1})\leq&\left(1-\frac{\gamma_{\pi}\lambda}{2}\right)D_{\psi}(\pi^*(\theta^t),\pi^t)-D_{\psi}(\pi^{t+1},\pi^t)\\&+\gamma_{\pi}\left\langle\nabla_{\pi}\mathcal{L}(\theta^t,\pi^*(\theta^t))-\nabla_{\pi}\mathcal{L}(\theta^t,\pi^t),\pi^t-\pi^{t+1}\right\rangle.
    \end{align*}
    Next, we apply Cauchy-Schwartz inequality to the scalar product and obtain
    \begin{align*}
        D_{\psi}(\pi^*(\theta^t),\pi^{t+1})\leq&\left(1-\frac{\gamma_{\pi}\lambda}{2}\right)D_{\psi}(\pi^*(\theta^t),\pi^t)-D_{\psi}(\pi^{t+1},\pi^t)\\&+\frac{\gamma_{\pi}\alpha}{2}\|\nabla_{\pi}\mathcal{L}(\theta^t,\pi^*(\theta^t))-\nabla_{\pi}\mathcal{L}(\theta^t,\pi^t)\|^2+\frac{\gamma_{\pi}}{2\alpha}\|\pi^t-\pi^{t+1}\|^2.
    \end{align*}
    Using $L_{\pi}$-smoothness of $\mathcal{L}(\theta,\pi)$ in $\pi$ (see Lemma \ref{lemma:p_lip}), we obtain
    \begin{align*}
        D_{\psi}(\pi^*(\theta^t),\pi^{t+1})\leq&\left(1-\frac{\gamma_{\pi}\lambda}{2}\right)D_{\psi}(\pi^*(\theta^t),\pi^t)-D_{\psi}(\pi^{t+1},\pi^t)\\&+\frac{\gamma_{\pi}\alpha L_{\pi}^2}{2}\|\pi^*(\theta^t)-\pi^t\|^2+\frac{\gamma_{\pi}}{2\alpha}\|\pi^t-\pi^{t+1}\|^2.
    \end{align*}
    Since $\psi$ is $1$-strongly convex (see Assumption \ref{ass:p}), we have
    \begin{align*}
        \frac{1}{2}\|\pi_1-\pi_2\|^2\leq D_{\psi}(\pi_1,\pi_2).
    \end{align*}
    Thus,
    \begin{align*}
        D_{\psi}(\pi^*(\theta^t),\pi^{t+1})\leq&\left(1-\frac{\gamma_{\pi}\lambda}{2}\right)D_{\psi}(\pi^*(\theta^t),\pi^t)-D_{\psi}(\pi^{t+1},\pi^t)\\&+\gamma_{\pi}\alpha L_{\pi}^2D_{\psi}(\pi^*(\theta^t),\pi^t)+\frac{\gamma_{\pi}}{\alpha}D_{\psi}(\pi^t,\pi^{t+1}).
    \end{align*}
    Choose $\alpha=\gamma_{\pi}$. We can derive
    \begin{align*}
        D_{\psi}(\pi^*(\theta^{t}),\pi^{t+1})\leq\left(1-\frac{\gamma_{\pi}\lambda}{2}+\gamma_{\pi}^2L_{\pi}^2\right)D_{\psi}(\pi^*(\theta^t),\pi^t).
    \end{align*}
    Since $\gamma_{\pi}=\nicefrac{\lambda}{4L_{\pi}^2}$, we have
    \begin{align}\label{eq:4_enh}
        D_{\psi}(\pi^*(\theta^{t}),\pi^{t+1})\leq\left(1-\frac{1}{16\kappa_{\pi}^2}\right)D_{\psi}(\pi^*(\theta^t),\pi^t).
    \end{align}
    Let us return to \eqref{eq:start_of_lemma_enh}. Note that
    \begin{align*}
        \nabla\psi(\pi^*(\theta^t))-\nabla\psi(\pi^{t+1})=\frac{1}{\lambda}\left(\nabla_{\pi}\mathcal{L}(\theta^t,\pi^{t+1})-\nabla_{\pi}\mathcal{L}(\theta^t,\pi^*(\theta^t))\right).
    \end{align*}
    Thus, there is
    \begin{align*}
        D_{\psi}(\pi^*(\theta^{t+1}),\pi^{t+1})=&D_{\psi}(\pi^*(\theta^{t+1}),\pi^*(\theta^{t}))+D_{\psi}(\pi^*(\theta^{t}),\pi^{t+1})\\&+\frac{1}{\lambda}\langle 
\nabla_{\pi}\mathcal{L}(\theta^t,\pi^{t+1})-\nabla_{\pi}\mathcal{L}(\theta^t,\pi^*(\theta^t)),\pi^*(\theta^{t+1})-\pi^*(\theta^{t}) \rangle\\\leq&D_{\psi}(\pi^*(\theta^{t+1}),\pi^*(\theta^{t}))+D_{\psi}(\pi^*(\theta^{t}),\pi^{t+1})\\&+\frac{\alpha L_{\pi}^2}{\lambda}D_{\psi}(\pi^*(\theta^t),\pi^{t+1})+\frac{1}{\lambda\alpha}D_{\psi}(\pi^*({\theta^{t+1}}),\pi^*(\theta^t)).
    \end{align*}
    Let us choose $\alpha=\nicefrac{\lambda^3}{32L_{\pi}^4}$. With such a choice, we have
    \begin{align*}
        D_{\psi}(\pi^*(\theta^{t+1}),\pi^{t+1})\leq33\kappa_{\pi}^4D_{\psi}(\pi^*({\theta^{t+1}}),\pi^*(\theta^t))+\left(1+\frac{1}{32\kappa_{\pi}^2}\right)D_{\psi}(\pi^*(\theta^t),\pi^{t+1}).
    \end{align*}
    To deal with $D_{\psi}(\pi^*(\theta^t),\pi^{t+1})$, we utilize \eqref{eq:4_enh}. As a result, we obtain
    \begin{align}\label{eq:7_enh}
        D_{\psi}(\pi^*(\theta^{t+1}),\pi^{t+1})\leq33\kappa_{\pi}^4D_{\psi}(\pi^*({\theta^{t+1}}),\pi^*(\theta^t))+\left(1-\frac{1}{32\kappa_{\pi}^2}\right)D_{\psi}(\pi^*(\theta^t),\pi^{t}).
    \end{align}
    The rest thing is to prove that the descent step does not dramatically change the distance between the optimal values of weights.
    Let us write down two optimality conditions:
    \begin{align*}
        &\langle \nabla_{\pi}\mathcal{L}(\theta^t,\pi^*(\theta^t)),\pi-\pi^*(\theta^t) \rangle\leq0,\\
        &\langle \nabla_{\pi}\mathcal{L}(\theta^{t+1},\pi^*(\theta^{t+1})),\pi-\pi^*(\theta^{t+1}) \rangle\leq0.
    \end{align*}
    Let us substitute $\pi=\pi^*(\theta^{t+1})$ into the first inequality and $\pi=\pi^*(\theta^{t})$ into the second one. When summing them up, we have
    \begin{align}\label{eq:5_enh}
        \langle \nabla_{\pi}\mathcal{L}(\theta^t,\pi^*(\theta^t))-\nabla_{\pi}\mathcal{L}(\theta^{t+1},\pi^*(\theta^{t+1})), \pi^*(\theta^{t+1})-\pi^*(\theta^t) \rangle\leq0.
    \end{align}
    On the other hand, we can take advantage of the strong concavity of the objective (see Lemma \ref{eq:lemma_1}) and write
    \begin{align}\label{eq:6_enh}
        &\langle \nabla_{\pi}\mathcal{L}(\theta^t,\pi^*(\theta^{t+1}))-\nabla_{\pi}\mathcal{L}(\theta^t,\pi^*(\theta^t)),\pi^*(\theta^{t+1})-\pi^*(\theta^{t}) \rangle\\&\leq-\frac{\lambda}{2}\left[D_{\psi}(\pi^*(\theta^t),\pi^*(\theta^{t+1}))+D_{\psi}(\pi^*(\theta^{t+1}),\pi^*(\theta^t))\right].
    \end{align}
    Combining \eqref{eq:5_enh} and \eqref{eq:6_enh}, we obtain
    \begin{align*}
        \frac{\lambda^2}{4}\left[D_{\psi}(\pi^*(\theta^t),\pi^*(\theta^{t+1}))+D_{\psi}(\pi^*(\theta^{t+1}),\pi^*(\theta^t))\right]^2\leq L^2\|\pi^*(\theta^{t+1})-\pi^*(\theta^{t})\|^2\cdot\|\theta^{t+1}-\theta^t\|^2.
    \end{align*}
    Here we can not apply the smoothness in $\pi$. Instead, we have to use the smoothness in $(\theta,\pi)$. Next, applying the strong convexity of distance generating function (Assumption \ref{ass:p}) and re-arranging terms, we obtain
    \begin{align*}
        D_{\psi}(\pi^*(\theta^t),\pi^*(\theta^{t+1}))+D_{\psi}(\pi^*(\theta^{t+1}),\pi^*(\theta^t))\leq&4\kappa^2\|\theta^{t+1}-\theta^t\|^2\leq4\gamma_{\theta}^2\kappa^2\|\nabla_{\theta}\mathcal{L}(\theta^t,\pi^t)\|^2.
    \end{align*}
    Next, we ass and subtract $\nabla\Phi(\theta^t)$ and apply Assumption \ref{ass:theta}. We obtain
    \begin{align*}
        D_{\psi}(\pi^*(\theta^t),\pi^*(\theta^{t+1}))+D_{\psi}(\pi^*(\theta^{t+1}),\pi^*(\theta^t))\leq&16\gamma_{\theta}^2\kappa^2L^2D_{\psi}(\pi^*(\theta^t),\pi^t)+8\gamma_{\theta}^2\kappa^2\|\nabla\Phi(\theta^t)\|^2.
    \end{align*}
    Thus, \eqref{eq:7_enh} transforms into 
    \begin{align*}
        D_{\psi}(\pi^*(\theta^{t+1}),\pi^{t+1})\leq\left(1-\frac{1}{32\kappa_{\pi}^2}+528\gamma_{\theta}^2\kappa_{\pi}^4\kappa^2L^2\right)D_{\psi}(\pi^*(\theta^t),\pi^{t})+264\gamma_{\theta}^2\kappa_{\pi}^4\kappa^2\|\nabla\Phi(\theta^t)\|^2.
    \end{align*}
    With $\gamma_{\theta}\leq\nicefrac{1}{184\kappa^3\kappa L}$, we obtain
    \begin{align*}
        D_{\psi}(\pi^*(\theta^{t+1}),\pi^{t+1})\leq\left(1-\frac{1}{64\kappa_{\pi}^2}\right)D_{\psi}(\pi^*(\theta^t),\pi^{t})+264\gamma_{\theta}^2\kappa_{\pi}^4\kappa^2\|\nabla\Phi(\theta^t)\|^2.
    \end{align*}
    This completes the proof.
\end{proof}
Next, we modify the main proof to obtain enhanced convergence.
\begin{theorem}\label{th:general_enh}.
    Consider the problem \eqref{eq:pinn_saddle} under Assumptions \ref{ass:theta}, \ref{ass:p}. Let $S=S=\triangle_1^{M-1}\cap B_{\|\cdot\|}(\mathcal{U},R)$. Then, Algorithm \ref{alg:bgda} with tuning 
    \begin{align*}
        \gamma_{\pi}=\frac{\lambda}{4L_{\pi}^2},\quad\gamma_{\theta}\leq\sqrt{\frac{43}{92*33792}}\frac{1}{\kappa_{\pi}^3\kappa L}
    \end{align*}
    requires
    \begin{align*}
        \mathcal{O}\left(\frac{\kappa L\Delta+L^2D_{\psi}(\pi^*(\theta^0),\pi^0)}{\varepsilon^2}\right)\text{ iterations}
    \end{align*}
    to achieve an arbitrary $\varepsilon$-solution, where $\varepsilon^2 = \frac{1}{T}\sum_{t=1}^{T-1}\|\nabla\Phi(\theta^t)\|^2$, $\Delta=\Phi(\theta^0)-\Phi(\theta^*)$. $\kappa=\nicefrac{L}{\lambda}$, $\kappa_{\pi}=\nicefrac{L_{\pi}}{\lambda}$.
\end{theorem}
\begin{proof}
    One can note that $\Phi$ is $3\kappa L$-smooth. Indeed,
    \begin{align*}
        \|\nabla\Phi(\theta_1)-\nabla\Phi(\theta_2)\|^2=&\|\nabla_{\theta}\mathcal{L}(\theta_1,\pi^*(\theta_1))-\nabla_{\theta}\mathcal{L}(\theta_2,\pi^*(\theta_2))\|^2\\\leq&  L^2\left[\|\theta_1-\theta_2\|^2+2D_{\psi}(\pi^*(\theta_1),\pi^*(\theta_2))\right]\leq L^2\left(1+4\kappa^2\right)\|\theta_1-\theta_2\|^2\\\leq&9\kappa^2L^2\|\theta_1-\theta_2\|^2.
    \end{align*}
    Thus, we can write
    \begin{align*}
        \Phi(\theta^{t+1})\leq&\Phi(\theta^t)+\langle \nabla\Phi(\theta^t), \theta^{t+1}-\theta^t \rangle+3\kappa L\|\theta^{t+1}-\theta^t\|^2\\\leq&\Phi(\theta^t)-\gamma_{\theta}\|\nabla\Phi(\theta^t)\|^2+3\gamma_{\theta}^2\kappa L\|\nabla_{\theta}\mathcal{L}(\theta^t,\pi^t)\|^2\\&+\gamma_{\theta}\langle \nabla\Phi(\theta^t)-\nabla_{\theta}\mathcal{L}(\theta^t,\pi^t),\nabla\Phi(\theta^t) \rangle\\\leq&\Phi(\theta^t)-\frac{\gamma_{\theta}}{2}\|\nabla\Phi(\theta^t)\|^2+3\gamma_{\theta}^2\kappa L\|\nabla_{\theta}\mathcal{L}(\theta^t,\pi^t)\|^2+\frac{\gamma_{\theta}}{2}\|\nabla\Phi(\theta^t)-\nabla_{\theta}\mathcal{L}(\theta^t,\pi^t)\|^2\\\leq&\Phi(\theta^t)-\left(\frac{\gamma_{\theta}}{2}- 6\gamma_{\theta}^2\kappa L\right)\|\nabla\Phi(\theta^t)\|^2+\left(\frac{\gamma_{\theta}}{2}+6\gamma_{\theta}^2\kappa L\right)\|\nabla\Phi(\theta^t)-\nabla_{\theta}\mathcal{L}(\theta^t,\pi^t)\|^2.
    \end{align*}
    Note that
    \begin{align*}
        -\left(\frac{\gamma_{\theta}}{2}- 6\gamma_{\theta}^2\kappa L\right)\leq-\frac{43\gamma_{\theta}}{92}.
    \end{align*}
    On the other hand,
    \begin{align*}
        \left(\frac{\gamma_{\theta}}{2}+6\gamma_{\theta}^2\kappa L\right)\leq\gamma_{\theta}.
    \end{align*}
    Thus, we have
    \begin{align*}
        \Phi(\theta^{t+1})\leq&\Phi(\theta^t)-\frac{43\gamma_{\theta}}{92}\|\nabla\Phi(\theta^t)\|^2+\gamma_{\theta}\|\nabla\Phi(\theta^t)-\nabla_{\theta}\mathcal{L}(\theta^t,\pi^t)\|^2\\\leq&\Phi(\theta^t)-\frac{43\gamma_{\theta}}{92}\|\nabla\Phi(\theta^t)\|^2+2\gamma_{\theta}L^2D_{\psi}(\pi^*({\theta^t}),\pi^t).
    \end{align*}
    Let us denote $\delta=1-\nicefrac{1}{64\kappa_{\pi}^2}$. Lemma \ref{lemma:distance_enh} transforms into
    \begin{align*}
        D_{\psi}(\pi^*(\theta^t),\pi^t)\leq\delta^tD_{\psi}(\pi^*(\theta^0),\pi^0)+264\gamma_{\theta}^2\kappa_{\pi}^4\kappa^2\sum_{j=0}^{t-1}\delta^{t-1-j}\|\nabla\Phi(\theta^j)\|^2.
    \end{align*}
    Hence,
    \begin{align*}
        \Phi(\theta^{t+1})\leq&\Phi(\theta^t)-\frac{43\gamma_{\theta}}{92}\|\nabla\Phi(\theta^t)\|^2+2\gamma_{\theta}L^2\delta^tD_{\psi}(\pi^*(\theta^0),\pi^0)\\&+528\gamma_{\theta}^3\kappa_{\pi}^4\kappa^2L^2\sum_{j=0}^{t-1}\delta^{t-1-j}\|\nabla\Phi(\theta^j)\|^2.
    \end{align*}
    Let us sum up over the iterates $t$ and obtain
    \begin{align*}
        \Phi(\theta^{T})\leq&\Phi(\theta^0)-\frac{43\gamma_{\theta}}{92}\sum_{t=1}^{T-1}\|\nabla\Phi(\theta^t)\|^2+2\gamma_{\theta}L^2\sum_{t=1}^{T-1}\delta^tD_{\psi}(\pi^*(\theta^0),\pi^0)\\&+528\gamma_{\theta}^3\kappa_{\pi}^4\kappa^2L^2\sum_{t=1}^{T-1}\sum_{j=0}^{t-1}\delta^{t-1-j}\|\nabla\Phi(\theta^j)\|^2.
    \end{align*}
    Next, we use the property of geometric progression and write
    \begin{align*}
        \Phi(\theta^{T})\leq&\Phi(\theta^0)-\frac{43\gamma_{\theta}}{92}\sum_{t=1}^{T-1}\|\nabla\Phi(\theta^t)\|^2+128\gamma_{\theta}\kappa_{\pi}^2L^2D_{\psi}(\pi^*(\theta^0),\pi^0)\\&+33792\gamma_{\theta}^3\kappa_{\pi}^6\kappa^2L^2\sum_{t=1}^{T-1}\|\nabla\Phi(\theta^t)\|^2.
    \end{align*}
    Choosing $\gamma_{\theta}\leq\sqrt{\frac{43}{92*33792}}\frac{1}{\kappa_{\pi}^3\kappa=L}$.
    Thus, we derive
    \begin{align*}
        \frac{1}{T}\sum_{t=1}^{T-1}\|\nabla\Phi(\theta^t)\|^2\leq\mathcal{O}\left(\frac{\kappa_{\pi}^3\kappa L\Delta_{\Phi}}{T}+\frac{\kappa_{\pi}^2L^2D_{\psi}(\pi^*(\theta^0),\pi^0)}{T}\right).
    \end{align*}
    Above we discussed that $\kappa_{\pi}$ is small, since not many equations appear in the PDEs systems. Thus, we can focus on $\kappa$ only and proceed to 
    \begin{align*}
        \frac{1}{T}\sum_{t=1}^{T-1}\|\nabla\Phi(\theta^t)\|^2\leq\mathcal{O}\left(\frac{\kappa L\Delta_{\Phi}}{T}+\frac{L^2D_{\psi}(\pi^*(\theta^0),\pi^0)}{T}\right).
    \end{align*}
    This finishes the proof.
\end{proof}

\section{Stochastic Setting}\label{ap:E}
In the current realities of machine learning, it is almost never possible to use all the data to compute a gradient. Motivated by this fact, we develop a stochastic theory for our scheme. Note that the computation $\nabla_{\pi}\mathcal{L}(\theta,\pi)$ does not need to perform backward. Therefore, we analyze the stochasticity in $\theta$ only. Consider a stochastic gradient $G_{\theta}(\theta^t,\pi^t,\xi)$ calculated from one randomly selected sample $\xi$. 
\begin{assumption}\label{ass:stoch}
    Stochastic oracle $G_{\theta}$ is unbiased and light-tailed, i.e.
    \begin{align*}
        \E_{\xi}\left[G_{\theta}(\theta,\pi,\xi)\right]=\nabla_{\theta}\mathcal{L}(\theta,\pi),~\E\left[\|G_{\theta}(\theta,\pi,\xi)-\nabla_{\theta}\mathcal{L}(\theta,\pi)\|^2\right]\leq\sigma^2,~\forall(\theta,\pi)\in\Rd\times S.
    \end{align*}
\end{assumption}
In our analysis, we rely on batching. Namely, we sample a subset of data points and use it to approximate the gradient.
\begin{algorithm}{\texttt{S-BGDA}} \label{alg:bgda_stoch}
\begin{algorithmic}[1]
    \State {\bfseries Input:} Starting point $(\theta^0,\pi^0)\in\Rd\times S$, number of iterations $T$
    \State {\bfseries Parameters:} Stepsizes $\gamma_{\theta},\gamma_{\pi} > 0$
    \For{$t=0,\ldots, T-1$}
    \State Draw a collection of i.i.d. data points $\{\xi_i^t\}_{i=1}^B$
    \State $\theta^{t+1}=\theta^t-\gamma_{\theta}\frac{1}{B}\sum_{i=1}^BG_{\theta}(\theta^t,\pi^t,\xi_i^t)$\label{line:descent_stoch} \Comment{Optimizer updates parameters}
    \State $\pi^{t+1}=\arg\min_{\pi\in S}\left\{ -\gamma_{\pi}\left\langle \nabla_{\pi}\mathcal{L}(\theta^t,\pi^t),\pi \right\rangle +D_{\psi}(\pi,\pi^t) \right\}$\label{line:ascent_stoch} \Comment{Optimizer updates weights}
    \EndFor
    \State {\bfseries Output:} $(\theta^T,\pi^T)$
\end{algorithmic}
\end{algorithm}
The main difference between Algorithm \ref{alg:bgda_stoch} and deterministic \texttt{BGDA} is the use of stochastic oracle call in Line \ref{line:descent_stoch}.
\begin{lemma}\label{lemma:distance_stoch}
    Consider the problem \eqref{eq:pinn_saddle} under Assumptions \ref{ass:theta}, \ref{ass:p}, \ref{ass:stoch}. Then, Algorithm \ref{alg:bgda_stoch} with tuning
    \begin{align*}
        \gamma_{\pi}=\frac{\lambda}{4L^2},\quad\gamma_{\theta}\leq\frac{1}{184\kappa^4L}
    \end{align*}
    produces such $\{(\theta^t,\pi^t)\}_{t=1}^T$, that
    \begin{align*}
        D_{\psi}(\pi^*(\theta^{t+1}),\pi^{t+1})\leq\left(1-\frac{1}{64\kappa^2}\right)D_{\psi}(\pi^*(\theta^t),\pi^{t})+264\gamma_{\theta}^2\kappa^6\|\nabla\Phi(\theta^t)\|^2+\frac{132\gamma_{\theta}^2\kappa^6\sigma^2}{B},
    \end{align*}
    where $\kappa=\nicefrac{L}{\lambda}$ is the condition number of $\mathcal{L}(\theta,\pi)$ in $\pi$.
\end{lemma}
\begin{proof}
    To begin, we use \eqref{eq:three-point} in the form
    \begin{align}\label{eq:start_of_lemma_stoch}
    \begin{split}
        D_{\psi}(\pi^*(\theta^{t+1}),\pi^{t+1})=&D_{\psi}(\pi^*(\theta^{t+1}),\pi^*(\theta^{t}))+D_{\psi}(\pi^*(\theta^{t}),\pi^{t+1})\\&+\langle 
\nabla\psi(\pi^*(\theta^{t}))-\nabla\psi(\pi^{t+1}),\pi^*(\theta^{t+1})-\pi^*(\theta^{t}) \rangle.
    \end{split}
    \end{align}
    Further, we write the optimality condition for Line \ref{line:ascent_stoch}:
    \begin{align*}
        \left\langle -\gamma_{\pi}\nabla_{\pi}\mathcal{L}(\theta^t,\pi^t)+[\nabla\psi(\pi^{t+1})-\nabla\psi(\pi^t)],\pi^*(\theta^t)-\pi^{t+1} \right\rangle\geq0.
    \end{align*}
    Applying \eqref{eq:three-point}, we obtain
    \begin{align*}
        -\gamma_{\pi}\left\langle\nabla_{\pi}\mathcal{L}(\theta^t,\pi^t),\pi^*(\theta^t)-\pi^{t+1}\right\rangle+D_{\psi}(\pi^*(\theta^t),\pi^t)-D_{\psi}(\pi^*(\theta^t),\pi^{t+1})-D_{\psi}(\pi^{t+1},\pi^t)\geq0.
    \end{align*}
    After re-arranging the terms, we get
    \begin{align}\label{eq:2_stoch}
        D_{\psi}(\pi^*(\theta^t),\pi^{t+1})\leq D_{\psi}(\pi^*(\theta^t),\pi^t)-D_{\psi}(\pi^{t+1},\pi^t)-\gamma_{\pi}\left\langle\nabla_{\pi}\mathcal{L}(\theta^t,\pi^t),\pi^*(\theta^t)-\pi^{t+1}\right\rangle.
    \end{align}
    Since $\pi^*(\theta^t)$ is the exact maximum of $\mathcal{L}(\theta^t,\pi)$ in $\pi$, there is another optimility condition
    \begin{align*}
        \gamma_{\pi}\left\langle \nabla_{\pi}\mathcal{L}(\theta^t,\pi^*(\theta^t)),\pi^*(\theta^t)-\pi \right\rangle\geq0.
    \end{align*}
    Substituting $\pi=\pi^{t+1}$ and summing it with \eqref{eq:2_stoch}, we derive
    \begin{align*}
        D_{\psi}(\pi^*(\theta^t),\pi^{t+1})\leq& D_{\psi}(\pi^*(\theta^t),\pi^t)-D_{\psi}(\pi^{t+1},\pi^t)\\&+\gamma_{\pi}\left\langle\nabla_{\pi}\mathcal{L}(\theta^t,\pi^*(\theta^t))-\nabla_{\pi}\mathcal{L}(\theta^t,\pi^t),\pi^*(\theta^t)-\pi^{t+1}\right\rangle\\\leq&D_{\psi}(\pi^*(\theta^t),\pi^t)-D_{\psi}(\pi^{t+1},\pi^t)\\&+\gamma_{\pi}\left\langle\nabla_{\pi}\mathcal{L}(\theta^t,\pi^*(\theta^t))-\nabla_{\pi}\mathcal{L}(\theta^t,\pi^t),\pi^*(\theta^t)-\pi^{t}\right\rangle\\&+\gamma_{\pi}\left\langle\nabla_{\pi}\mathcal{L}(\theta^t,\pi^*(\theta^t))-\nabla_{\pi}\mathcal{L}(\theta^t,\pi^t),\pi^t-\pi^{t+1}\right\rangle.
    \end{align*}
    Now, we are going to utilize the strong concavity of $\mathcal{L}(\theta,\pi)$ in $\pi$:
    \begin{align*}
        \gamma_{\pi}\left\langle\nabla_{\pi}\mathcal{L}(\theta^t,\pi^*(\theta^t))-\nabla_{\pi}\mathcal{L}(\theta^t,\pi^t),\pi^*(\theta^t)-\pi^{t}\right\rangle\leq\frac{-\gamma_{\pi}\lambda}{2}D_{\psi}(\pi^*(\theta^t),\pi^t).
    \end{align*}
    Thus, we have
    \begin{align*}
        D_{\psi}(\pi^*(\theta^t),\pi^{t+1})\leq&\left(1-\frac{\gamma_{\pi}\lambda}{2}\right)D_{\psi}(\pi^*(\theta^t),\pi^t)-D_{\psi}(\pi^{t+1},\pi^t)\\&+\gamma_{\pi}\left\langle\nabla_{\pi}\mathcal{L}(\theta^t,\pi^*(\theta^t))-\nabla_{\pi}\mathcal{L}(\theta^t,\pi^t),\pi^t-\pi^{t+1}\right\rangle.
    \end{align*}
    Next, we apply Cauchy-Schwartz inequality to the scalar product and obtain
    \begin{align*}
        D_{\psi}(\pi^*(\theta^t),\pi^{t+1})\leq&\left(1-\frac{\gamma_{\pi}\lambda}{2}\right)D_{\psi}(\pi^*(\theta^t),\pi^t)-D_{\psi}(\pi^{t+1},\pi^t)\\&+\frac{\gamma_{\pi}\alpha}{2}\|\nabla_{\pi}\mathcal{L}(\theta^t,\pi^*(\theta^t))-\nabla_{\pi}\mathcal{L}(\theta^t,\pi^t)\|^2+\frac{\gamma_{\pi}}{2\alpha}\|\pi^t-\pi^{t+1}\|^2.
    \end{align*}
    Using $L$-smoothness of $\mathcal{L}$ (see Assumption \ref{ass:theta}), we obtain
    \begin{align*}
        D_{\psi}(\pi^*(\theta^t),\pi^{t+1})\leq&\left(1-\frac{\gamma_{\pi}\lambda}{2}\right)D_{\psi}(\pi^*(\theta^t),\pi^t)-D_{\psi}(\pi^{t+1},\pi^t)\\&+\frac{\gamma_{\pi}\alpha L^2}{2}\|\pi^*(\theta^t)-\pi^t\|^2+\frac{\gamma_{\pi}}{2\alpha}\|\pi^t-\pi^{t+1}\|^2.
    \end{align*}
    Since $\psi$ is $1$-strongly convex (see Assumption \ref{ass:p}), we have
    \begin{align*}
        \frac{1}{2}\|\pi_1-\pi_2\|^2\leq D_{\psi}(\pi_1,\pi_2).
    \end{align*}
    Thus,
    \begin{align*}
        D_{\psi}(\pi^*(\theta^t),\pi^{t+1})\leq&\left(1-\frac{\gamma_{\pi}\lambda}{2}\right)D_{\psi}(\pi^*(\theta^t),\pi^t)-D_{\psi}(\pi^{t+1},\pi^t)\\&+\gamma_{\pi}\alpha L^2D_{\psi}(\pi^*(\theta^t),\pi^t)+\frac{\gamma_{\pi}}{\alpha}D_{\psi}(\pi^t,\pi^{t+1}).
    \end{align*}
    Choose $\alpha=\gamma_{\pi}$. We can derive
    \begin{align*}
        D_{\psi}(\pi^*(\theta^{t}),\pi^{t+1})\leq\left(1-\frac{\gamma_{\pi}\lambda}{2}+\gamma_{\pi}^2L^2\right)D_{\psi}(\pi^*(\theta^t),\pi^t).
    \end{align*}
    Since $\gamma_{\pi}=\nicefrac{\lambda}{4L^2}$, we have
    \begin{align}\label{eq:4_stoch}
        D_{\psi}(\pi^*(\theta^{t}),\pi^{t+1})\leq\left(1-\frac{1}{16\kappa^2}\right)D_{\psi}(\pi^*(\theta^t),\pi^t).
    \end{align}
    Let us return to \eqref{eq:start_of_lemma_stoch}. Note that
    \begin{align*}
        \nabla\psi(\pi^*(\theta^t))-\nabla\psi(\pi^{t+1})=\frac{1}{\lambda}\left(\nabla_{\pi}\mathcal{L}(\theta^t,\pi^{t+1})-\nabla_{\pi}\mathcal{L}(\theta^t,\pi^*(\theta^t))\right).
    \end{align*}
    Thus, there is
    \begin{align*}
        D_{\psi}(\pi^*(\theta^{t+1}),\pi^{t+1})=&D_{\psi}(\pi^*(\theta^{t+1}),\pi^*(\theta^{t}))+D_{\psi}(\pi^*(\theta^{t}),\pi^{t+1})\\&+\frac{1}{\lambda}\langle 
\nabla_{\pi}\mathcal{L}(\theta^t,\pi^{t+1})-\nabla_{\pi}\mathcal{L}(\theta^t,\pi^*(\theta^t)),\pi^*(\theta^{t+1})-\pi^*(\theta^{t}) \rangle\\\leq&D_{\psi}(\pi^*(\theta^{t+1}),\pi^*(\theta^{t}))+D_{\psi}(\pi^*(\theta^{t}),\pi^{t+1})\\&+\frac{\alpha L^2}{\lambda}D_{\psi}(\pi^*(\theta^t),\pi^{t+1})+\frac{1}{\lambda\alpha}D_{\psi}(\pi^*({\theta^{t+1}}),\pi^*(\theta^t)).
    \end{align*}
    Let us choose $\alpha=\nicefrac{\lambda^3}{32L^4}$. With such a choice, we have
    \begin{align*}
        D_{\psi}(\pi^*(\theta^{t+1}),\pi^{t+1})\leq33\kappa^4D_{\psi}(\pi^*({\theta^{t+1}}),\pi^*(\theta^t))+\left(1+\frac{1}{32\kappa^2}\right)D_{\psi}(\pi^*(\theta^t),\pi^{t+1}).
    \end{align*}
    To deal with $D_{\psi}(\pi^*(\theta^t),\pi^{t+1})$, we utilize \eqref{eq:4_stoch}. As a result, we obtain
    \begin{align}\label{eq:7_stoch}
        D_{\psi}(\pi^*(\theta^{t+1}),\pi^{t+1})\leq33\kappa^4D_{\psi}(\pi^*({\theta^{t+1}}),\pi^*(\theta^t))+\left(1-\frac{1}{32\kappa^2}\right)D_{\psi}(\pi^*(\theta^t),\pi^{t}).
    \end{align}
    The rest thing is to prove that the descent step does not dramatically change the distance between the optimal values of weights.
    Let us write down two optimality conditions:
    \begin{align*}
        &\langle \nabla_{\pi}\mathcal{L}(\theta^t,\pi^*(\theta^t)),\pi-\pi^*(\theta^t) \rangle\leq0,\\
        &\langle \nabla_{\pi}\mathcal{L}(\theta^{t+1},\pi^*(\theta^{t+1})),\pi-\pi^*(\theta^{t+1}) \rangle\leq0.
    \end{align*}
    Let us substitute $\pi=\pi^*(\theta^{t+1})$ into the first inequality and $\pi=\pi^*(\theta^{t})$ into the second one. When summing them up, we have
    \begin{align}\label{eq:5_stoch}
        \langle \nabla_{\pi}\mathcal{L}(\theta^t,\pi^*(\theta^t))-\nabla_{\pi}\mathcal{L}(\theta^{t+1},\pi^*(\theta^{t+1})), \pi^*(\theta^{t+1})-\pi^*(\theta^t) \rangle\leq0.
    \end{align}
    On the other hand, we can take advantage of the strong concavity of the objective (see Lemma \ref{eq:lemma_1}) and write
    \begin{align}\label{eq:6_stoch}
        &\langle \nabla_{\pi}\mathcal{L}(\theta^t,\pi^*(\theta^{t+1}))-\nabla_{\pi}\mathcal{L}(\theta^t,\pi^*(\theta^t)),\pi^*(\theta^{t+1})-\pi^*(\theta^{t}) \rangle\\&\leq-\frac{\lambda}{2}\left[D_{\psi}(\pi^*(\theta^t),\pi^*(\theta^{t+1}))+D_{\psi}(\pi^*(\theta^{t+1}),\pi^*(\theta^t))\right].
    \end{align}
    Combining \eqref{eq:5_stoch} and \eqref{eq:6_stoch}, we obtain
    \begin{align*}
        \frac{\lambda^2}{4}\left[D_{\psi}(\pi^*(\theta^t),\pi^*(\theta^{t+1}))+D_{\psi}(\pi^*(\theta^{t+1}),\pi^*(\theta^t))\right]^2\leq L^2\|\pi^*(\theta^{t+1})-\pi^*(\theta^{t})\|^2\|\theta^{t+1}-\theta^t\|^2.
    \end{align*}
    Re-arranging the terms and substituting Line \ref{line:descent_stoch}, we derive
    \begin{align*}
        \left[D_{\psi}(\pi^*(\theta^t),\pi^*(\theta^{t+1}))+D_{\psi}(\pi^*(\theta^{t+1}),\pi^*(\theta^t))\right]\leq& 4\kappa^2\|\theta^{t+1}-\theta^t\|^2\\\leq&4\gamma_{\theta}^2\kappa^2\left\| \frac{1}{B}\sum_{i=1}^BG_{\theta}(\theta^t,\pi^t,\xi_i^t) \right\|^2.
    \end{align*}
    After adding and subtracting $\nabla_{\theta}\mathcal{L}(\theta^t,\pi^t)$, we have
    \begin{align*}
        D_{\psi}(\pi^*(\theta^{t+1}),\pi^*(\theta^t))\leq&4\gamma_{\theta}^2\kappa^2\left\| \nabla_{\theta}\mathcal{L}(\theta^t,\pi^t) \right\|^2+4\gamma_{\theta}^2\kappa^2\left\|\nabla_{\theta}\mathcal{L}(\theta^t,\pi^t)-\frac{1}{B}\sum_{i=1}^BG_{\theta}(\theta^t,\pi^t,\xi_i^t)\right\|^2.
    \end{align*}
    Let us take an expectation and derive
    \begin{align*}
        \E D_{\psi}(\pi^*(\theta^{t+1}),\pi^*(\theta^t))\leq&\E8\gamma_{\theta}^2\kappa^2\|\nabla\Phi(\theta^t)\|^2+8\gamma_{\theta}^2\kappa^2\|\nabla_{\theta}\mathcal{L}(\theta^t,\pi^t)-\nabla\Phi(\theta^t)\|^2+\frac{4\gamma_{\theta}^2\kappa^2\sigma^2}{B}\\\leq&\E8\gamma_{\theta}^2\kappa^2\|\nabla\Phi(\theta^t)\|^2+16\gamma_{\theta}^2\kappa^2L^2D_{\psi}(\pi^*(\theta^t),\pi^t)+\frac{4\gamma_{\theta}^2\kappa^2\sigma^2}{B}.
    \end{align*}
    Thus, \eqref{eq:7_stoch} transforms into 
    \begin{align*}
        \E D_{\psi}(\pi^*(\theta^{t+1}),\pi^{t+1})\leq&\E\left(1-\frac{1}{32\kappa^2}+528\gamma_{\theta}^2\kappa^6L^2\right)D_{\psi}(\pi^*(\theta^t),\pi^{t})+264\gamma_{\theta}^2\kappa^6\|\nabla\Phi(\theta^t)\|^2\\&+\frac{132\gamma_{\theta}^2\kappa^6\sigma^2}{B}.
    \end{align*}
    With $\gamma_{\theta}\leq\nicefrac{1}{184\kappa^4L}$, we obtain
    \begin{align*}
        \E D_{\psi}(\pi^*(\theta^{t+1}),\pi^{t+1})\leq\E\left(1-\frac{1}{64\kappa^2}\right)D_{\psi}(\pi^*(\theta^t),\pi^{t})+264\gamma_{\theta}^2\kappa^6\|\nabla\Phi(\theta^t)\|^2+\frac{132\gamma_{\theta}^2\kappa^6\sigma^2}{B}.
    \end{align*}
    This completes the proof.
\end{proof}
Now let us proceed to the convergence proof for Algorithm \ref{alg:bgda_stoch}.
\begin{theorem}\label{th:general_stoch}
    Consider the problem \eqref{eq:pinn_saddle} under Assumptions \ref{ass:theta}, \ref{ass:p}, \ref{ass:stoch}. Then, Algorithm \ref{alg:bgda} with tuning 
    \begin{align*}
        \gamma_{\pi}=\frac{\lambda}{4L^2},\quad\gamma_{\theta}\leq\sqrt{\frac{43}{92*33792}}\frac{1}{\kappa^4L},\quad B=\max\left\{ 
1,\frac{\kappa^{\nicefrac{3}{2}}}{\varepsilon^2} \right\}
    \end{align*}
    requires
    \begin{align*}
        \mathcal{O}\left(\frac{\kappa^4L\Delta+\kappa^2L^2D_{\psi}(\pi^*(\theta^0),\pi^0)+\kappa^{\nicefrac{3}{2}}\sigma^2}{\varepsilon^2}\right)\text{ iterations}
    \end{align*}
    to achieve an arbitrary $\varepsilon$-solution, where $\varepsilon^2 = \frac{1}{T}\sum_{t=1}^{T-1}\|\nabla\Phi(\theta^t)\|^2$, $\Delta=\Phi(\theta^0)-\Phi(\theta^*)$. $\kappa=\nicefrac{L}{\lambda}$.
\end{theorem}
\begin{proof}
    One can note that $\Phi$ is $3\kappa L$-smooth. Indeed,
    \begin{align*}
        \|\nabla\Phi(\theta_1)-\nabla\Phi(\theta_2)\|^2=&\|\nabla_{\theta}\mathcal{L}(\theta_1,\pi^*(\theta_1))-\nabla_{\theta}\mathcal{L}(\theta_2,\pi^*(\theta_2))\|^2\\\leq&  L^2\left[\|\theta_1-\theta_2\|^2+2D_{\psi}(\pi^*(\theta_1),\pi^*(\theta_2))\right]\leq L^2\left(1+4\kappa^2\right)\|\theta_1-\theta_2\|^2\\\leq&9\kappa^2L^2\|\theta_1-\theta_2\|^2.
    \end{align*}
    Thus, we can write
    \begin{align*}
        \Phi(\theta^{t+1})\leq&\Phi(\theta^t)+\langle \nabla\Phi(\theta^t), \theta^{t+1}-\theta^t \rangle+3\kappa L\|\theta^{t+1}-\theta^t\|^2\\=&\Phi(\theta^t)-\gamma_{\theta}\left\langle 
\nabla\Phi(\theta^t),\frac{1}{B}\sum_{i=1}^BG_{\theta}(\theta^t,\pi^t,\xi_i^t) \right\rangle+3\gamma_{\theta}^2\kappa L\left\|\frac{1}{B}\sum_{i=1}^BG_{\theta}(\theta^t,\pi^t,\xi_i^t)\right\|^2\\=&\Phi(\theta^t)-\gamma_{\theta}\|\nabla\Phi(\theta^t)\|^2+\gamma_{\theta}\left\langle \nabla\Phi(\theta^t),\nabla\Phi(\theta^t)-\frac{1}{B}\sum_{i=1}^BG_{\theta}(\theta^t,\pi^t,\xi_i^t) \right\rangle\\&+6\gamma_{\theta}^2\kappa L\|\nabla_{\theta}\mathcal{L}(\theta^t,\pi^t)\|^2+6\gamma_{\theta}^2\kappa L\left\|\nabla_{\theta}\mathcal{L}(\theta^t,\pi^t)-\frac{1}{B}\sum_{i=1}^BG_{\theta}(\theta^t,\pi^t,\xi_i^t)\right\|^2.
    \end{align*}
    Consider an expectation. We have
    \begin{align*}
        \E\Phi(\theta^{t+1})\leq&\E\Phi(\theta^t)-\gamma_{\theta}\|\nabla\Phi(\theta^t)\|^2+\gamma_{\theta}\left\langle \nabla\Phi(\theta^t),\nabla\Phi(\theta^t)-\nabla_{\theta}\mathcal{L}(\theta^t,\pi^t) \right\rangle\\&+6\gamma_{\theta}^2\kappa L\|\nabla_{\theta}\mathcal{L}(\theta^t,\pi^t)\|^2+6\gamma_{\theta}^2\kappa L\sigma^2\\\leq&\E\Phi(\theta^t)-\left(\frac{\gamma_{\theta}}{2}- 12\gamma_{\theta}^2\kappa L\right)\|\nabla\Phi(\theta^t)\|^2\\&+\left(\frac{\gamma_{\theta}}{2}+12\gamma_{\theta}^2\kappa L\right)\|\nabla\Phi(\theta^t)-\nabla_{\theta}\mathcal{L}(\theta^t,\pi^t)\|^2+\frac{6\gamma_{\theta}^2\kappa L\sigma^2}{B}.
    \end{align*}
    Note that
    \begin{align*}
        -\left(\frac{\gamma_{\theta}}{2}- 12\gamma_{\theta}^2\kappa L\right)\leq-\frac{43\gamma_{\theta}}{92}.
    \end{align*}
    On the other hand,
    \begin{align*}
        \left(\frac{\gamma_{\theta}}{2}+12\gamma_{\theta}^2\kappa L\right)\leq\gamma_{\theta}.
    \end{align*}
    Thus, we have
    \begin{align*}
        \E\Phi(\theta^{t+1})\leq&\E\Phi(\theta^t)-\frac{43\gamma_{\theta}}{92}\|\nabla\Phi(\theta^t)\|^2+\gamma_{\theta}\|\nabla\Phi(\theta^t)-\nabla_{\theta}\mathcal{L}(\theta^t,\pi^t)\|^2+6\gamma_{\theta}^2\kappa L\sigma^2\\\leq&\E\Phi(\theta^t)-\frac{43\gamma_{\theta}}{92}\|\nabla\Phi(\theta^t)\|^2+2\gamma_{\theta}L^2D_{\psi}(\pi^*({\theta^t}),\pi^t)+\frac{6\gamma_{\theta}^2\kappa L\sigma^2}{B}.
    \end{align*}
    Let us denote $\delta=1-\nicefrac{1}{64\kappa^2}$. Lemma \ref{lemma:distance_stoch} transforms into
    \begin{align*}
        \E D_{\psi}(\pi^*(\theta^t),\pi^t)\leq&\E\delta^tD_{\psi}(\pi^*(\theta^0),\pi^0)+264\gamma_{\theta}^2\kappa^6\sum_{j=0}^{t-1}\delta^{t-1-j}\|\nabla\Phi(\theta^j)\|^2\\&+\sum_{j=0}^{t-1}\delta^{t-1-j}\frac{132\gamma_{\theta}^2\kappa^6\sigma^2}{B}.
    \end{align*}
    Hence,
    \begin{align*}
        \Phi(\theta^{t+1})\leq&\Phi(\theta^t)-\frac{43\gamma_{\theta}}{92}\|\nabla\Phi(\theta^t)\|^2+2\gamma_{\theta}L^2\delta^tD_{\psi}(\pi^*(\theta^0),\pi^0)\\&+528\gamma_{\theta}^3\kappa^6L^2\sum_{j=0}^{t-1}\delta^{t-1-j}\|\nabla\Phi(\theta^j)\|^2+\frac{6\gamma_{\theta}^2\kappa L\sigma^2}{B}\\&+\sum_{j=0}^{t-1}\delta^{t-1-j}\frac{264\gamma_{\theta}^3\kappa^6L^2\sigma^2}{B}.
    \end{align*}
    Let us sum up over the iterates $t$ and obtain
    \begin{align*}
        \Phi(\theta^{T})\leq&\Phi(\theta^0)-\frac{43\gamma_{\theta}}{92}\sum_{t=1}^{T-1}\|\nabla\Phi(\theta^t)\|^2+2\gamma_{\theta}L^2\sum_{t=1}^{T-1}\delta^tD_{\psi}(\pi^*(\theta^0),\pi^0)\\&+528\gamma_{\theta}^3\kappa^6L^2\sum_{t=1}^{T-1}\sum_{j=0}^{t-1}\delta^{t-1-j}\|\nabla\Phi(\theta^j)\|^2+\sum_{t=1}^{T-1}\frac{6\gamma_{\theta}^2\kappa L\sigma^2}{B}\\&+\sum_{t=1}^{T-1}\sum_{j=0}^{t-1}\delta^{t-1-j}\frac{264\gamma_{\theta}^3\kappa^6L^2\sigma^2}{B}.
    \end{align*}
    Next, we use the property of geometric progression and write
    \begin{align*}
        \Phi(\theta^{T})\leq&\Phi(\theta^0)-\frac{43\gamma_{\theta}}{92}\sum_{t=1}^{T-1}\|\nabla\Phi(\theta^t)\|^2+128\gamma_{\theta}\kappa^2L^2D_{\psi}(\pi^*(\theta^0),\pi^0)\\&+33792\gamma_{\theta}^3\kappa^8L^2\sum_{t=1}^{T-1}\|\nabla\Phi(\theta^t)\|^2+\frac{6T\gamma_{\theta}^2\kappa L\sigma^2}{B}+\frac{16896T\gamma_{\theta}^3\kappa^8L^2\sigma^2}{B}.
    \end{align*}
    Since $\gamma_{\theta}\leq\frac{1}{184\kappa^4L}$, we can estimate this as
    \begin{align*}
        \Phi(\theta^{T})\leq&\Phi(\theta^0)-\frac{43\gamma_{\theta}}{92}\sum_{t=1}^{T-1}\|\nabla\Phi(\theta^t)\|^2+128\gamma_{\theta}\kappa^2L^2D_{\psi}(\pi^*(\theta^0),\pi^0)\\&+33792\gamma_{\theta}^3\kappa^8L^2\sum_{t=1}^{T-1}\|\nabla\Phi(\theta^t)\|^2+\frac{\gamma_{\theta}T\sigma^2}{B\kappa^3}+\frac{92\gamma_{\theta}T\sigma^2}{B}.
    \end{align*}
    Choosing $\gamma_{\theta}\leq\sqrt{\frac{43}{92*33792}}\frac{1}{\kappa^4L}$, we derive
    \begin{align*}
        \frac{1}{T}\sum_{t=1}^{T-1}\|\nabla\Phi(\theta^t)\|^2\leq\mathcal{O}\left(\frac{\kappa^4L\Delta_{\Phi}}{T}+\frac{\kappa^2L^2D_{\psi}(\pi^*(\theta^0),\pi^0)}{T}+\frac{\sigma^2}{B\kappa^3}+\frac{92\sigma^2}{B}\right).
    \end{align*}
    Let us choose $B=\nicefrac{T}{\kappa^{\nicefrac{3}{2}}}$ and obtain 
    \begin{align*}
        \frac{1}{T}\sum_{t=1}^{T-1}\|\nabla\Phi(\theta^t)\|^2\leq\mathcal{O}\left(\frac{\kappa^4L\Delta_{\Phi}}{T}+\frac{\kappa^2L^2D_{\psi}(\pi^*(\theta^0),\pi^0)}{T}+\frac{\kappa^{\nicefrac{3}{2}}\sigma^2}{T}\right).
    \end{align*}
    This finishes the proof.
\end{proof}
Note that the same reasoning could be done for the special case of a regularized simplex. Then we would obtain improved rates.
\end{appendixpart}
\end{document}